%% file: paper.tex
\documentclass[twocolumn, 10pt]{article}

\usepackage{times}      
\usepackage[utf8]{inputenc}
\usepackage[T1]{fontenc}

\usepackage[left=0.8in, right=0.8in, top=1in, bottom=1in, columnsep=0.25in]{geometry}
\usepackage{algorithm}
\usepackage{algpseudocode}
\usepackage{amsmath}
\usepackage{amsthm}
\usepackage{booktabs}
\usepackage{amssymb}
\usepackage{graphicx}
\usepackage[table]{xcolor}
\usepackage{etoc}
\usepackage{titletoc}
\usepackage[round]{natbib}
\usepackage{xcolor}
\usepackage[
  colorlinks=true,
  linkcolor=blue!50!black,
  citecolor=blue!50!black,
  urlcolor=blue!50!black,
]{hyperref}

\usepackage{tikz}
\usepackage{subcaption}
\usetikzlibrary{arrows.meta,calc,positioning}

\definecolor{Teal}{rgb}{0.0, 0.5, 0.5}
\makeatletter
\renewcommand{\@fnsymbol}[1]{%
  \ifcase#1\or $\star$\or \dagger\or \ddagger\or
  \mathsection\or \mathparagraph\or \|\or **\or \dagger\dagger\or \ddagger\ddagger
  \else\@ctrerr\fi}
\makeatother
\input{defs}

\usepackage{fancyhdr}
\fancypagestyle{firstpage}{
  \fancyhf{}
  \fancyfoot[C]{\thepage}
  
}

\newenvironment{narrowabstract}
{
  \begin{center}
    \bfseries \large Abstract
  \end{center}
  \vspace{-0.5em}
  \begin{list}{}{
    \setlength{\leftmargin}{0.2in} 
    \setlength{\rightmargin}{0.2in}
    \setlength{\listparindent}{1em}
    \setlength{\itemindent}{0pt}
    \setlength{\parsep}{\parskip}
  }
  \item[] 
}
{
  \end{list}
  \vspace{1em}
}

\begin{document}

\renewcommand*{\thefootnote}{\fnsymbol{footnote}}
\twocolumn[
  \begin{@twocolumnfalse}
  
    \hrule height 4pt 
    \vspace{0.15in} 
    
    \begin{center}
      {\LARGE \bfseries Bandits in Flux: Adversarial Constraints in Dynamic Environments \par}
    \end{center}
    \vspace{0.15in}
    
    \hrule height 1pt 
    \vspace{0.25in}
    
    \begin{center}
      {\bfseries\large Tareq Si Salem\footnotemark[1]} 
    \end{center}
    \vspace{0.3in} 
    
  \end{@twocolumnfalse}
]


\footnotetext[1]{Paris Research Center, Huawei Technologies France}

\renewcommand*{\thefootnote}{\arabic{footnote}}
\setcounter{footnote}{0}

\thispagestyle{firstpage}

\begin{narrowabstract}
We investigate the challenging problem of adversarial multi-armed bandits operating under time-varying constraints, a scenario motivated by numerous real-world applications. To address this complex setting, we propose a novel primal-dual algorithm that extends online mirror descent through the incorporation of suitable gradient estimators and effective constraint handling. We provide theoretical guarantees establishing sublinear dynamic regret and sublinear constraint violation for our proposed policy. Our algorithm achieves state-of-the-art performance in terms of both regret and constraint violation. Empirical evaluations demonstrate the superiority of our approach.
\end{narrowabstract}


\section{Introduction}
The multi-armed bandit (MAB) problem is a canonical framework within the domain of online decision-making. At each time step $t$, a decision-maker selects an action $a_t$ from a finite action set $\A$. Subsequently, the environment reveals the corresponding cost $f_t(a_t)$. The objective is to minimize the cumulative incurred cost over time. A fundamental challenge in the MAB setting is the exploration-exploitation trade-off, which requires balancing the acquisition of new information about the environment with the exploitation of current knowledge to make optimal decisions.

While the foundational multi-armed bandit problem has been extensively explored for decades~\citep{auer2002finite,katehakis1987multi, bubeck2012regret}, its relevance persists in contemporary computing systems. Its applications encompass a wide range of domains, including recommendation systems~\citep{li2010contextual, li2011unbiased}, auction design~\citep{avadhanula2021stochastic, nguyen2020bandit,gao2021auction}, resource allocation~\citep{fu2022restless, verma2019censored}, and network routing and scheduling~\citep{gyorgy2007line,li2020multi, scheduling2, huang2023queue}.

In this work, we focus on a variant of the MAB problem characterized by additional soft constraints. We introduce an additional challenge by imposing soft constraints on the decision-maker. Specifically, when selecting an action at time $t$, the environment reveals a corresponding constraint $g_t(a_t)$. While momentary constraint violations are permissible, we require long-term constraint satisfaction, ensuring vanishing time-averaged constraint violation.

This formulation is motivated by a variety of real-world applications~\citep{shi2022bayesian,zhou2022kernelized,deng2022interference}. For instance, in the optimization of machine learning model hyperparameters~\citep{joy2016hyperparameter, zhou2022kernelized}, we often aim to minimize validation error while simultaneously ensuring a sufficiently short prediction time for the learned model. Another example arises in wireless communications, where maximizing throughput under energy constraints necessitates the management of a cumulative soft energy consumption limit~\citep{orhan2012throughput}. Beyond these, related constrained bandit formulations arise in optimizing user engagement under fairness constraints~\citep{hadiji2024diversity}, energy-aware client selection in federated learning systems~\citep{zhu2024client}, budget-constrained ad allocation~\citep{li2022contextual}, and fair resource distribution across competing tasks~\citep{wang2024online}. In each of these domains, even rare constraint violations can have significant negative impacts, motivating algorithms that explicitly control long-run feasibility while learning from partial feedback.

We make the following contributions:
\begin{itemize}
    \item We provide a methodology for leveraging the online mirror descent framework for a new variant of the non-stationary MAB problem with soft time-varying constraints. 
    \item To our knowledge, the proposed algorithm achieves the tightest known theoretical guarantees for dynamic regret and soft constraint violation in the context of MAB with soft constraints. Specifically, we derive an improved dynamic regret bound of  $\tilde{\mathcal{O}}\parentheses{\min\set{\sqrt{P_T T}, V_T^{1/3} T^{2/3}}}$ and constraint violation of $\tilde{\mathcal{O}}\parentheses{\sqrt{T}}$, surpassing the state-of-the-art results in~\citep{deng2022interference} of $\tilde{\mathcal{O}}\parentheses{V^{1/4}_T T^{3/4}}$ and $\tilde{\mathcal{O}} \parentheses{V_T^{1/4} T^{3/4}}$, respectively.
    \item We empirically validate the superiority of our approach through extensive experiments.
\end{itemize}

The remainder of the paper is organized as follows. We present related work and problem formulation in Section~\ref{s:related_work} and Section~\ref{s:problem_formulation}, respectively. We state our main results in Section~\ref{s:main_result}. Our experimental results are in Section~\ref{s:experimental_results}; we conclude in Section~\ref{s:conclusion}.
\begin{table*}[t!]
\centering
\resizebox{\textwidth}{!}{%
\begin{tabular}{l|c|c|c|c|c}
\hline
\textbf{Paper} & \textbf{Regularity used} 
& \textbf{Regret} (Asm.~\ref{asm:regularity}) 
& \textbf{Violation} (Asm.~\ref{asm:regularity})  
& \textbf{Regret} (Agnostic)  
& \textbf{Violation} (Agnostic) \\
\hline
\rowcolor{gray!10}
\multicolumn{6}{c}{\textbf{Static Unconstrained MAB}} \\
\citep{auer2002finite} & none 
& $\mathcal{O}(\sqrt{T})$ & N/A & N/A & N/A \\
\hline
\rowcolor{gray!10}
\multicolumn{6}{c}{\textbf{Dynamic Unconstrained MAB}} \\
\hline
\citep{zhao2021bandit} & $P_T$ 
& $T^{3/4}(1+P_T)^{1/2}$ & N/A 
& $T^{3/4}(1+P_T)^{1/2}$ & N/A \\
\hline
\citep{deng2022weighted} & $V_T$ 
& $V_T^{1/4} T^{3/4}$ & N/A 
& $V_T^{1/4} T^{3/4}$ & N/A \\
\hline
\citep{kim2025adversarialbanditsarbitrarystrategies} & $P_T$ 
& \xmark & N/A 
& $P_T \sqrt{T}$ & N/A \\
\hline
\citep{chen2025bridging} & both 
& $\min\set{V_T^{1/3} T^{2/3}, \sqrt{P_T T}}$ & N/A 
& \xmark & N/A \\
\hline
\rowcolor{gray!10}
\multicolumn{6}{c}{\textbf{Dynamic Constrained MAB}} \\
\hline
\citep{cao2018online} & $P_T$ 
& $\sqrt{P_T T}$ & $\parentheses{P_T}^{1/4} T^{3/4}$ 
& \xmark & \xmark \\
\hline
\citep{deng2022interference} & $V_T$  & $V_T^{1/4} T^{3/4}$ & $V_T^{1/4} T^{3/4}$
& $V_T T^{3/4}$ & $V_T T^{3/4}$ 
 \\
\hline
\rowcolor{teal!5}
This work & both 
& $\min\set{V_T^{1/3} T^{2/3}, \sqrt{P_T T}}$ & $\sqrt{T}$ 
& $ \min\set{\sqrt{P_T} T^{2/3},\, V_T^{1/3} T^{7/9}}$ & $\sqrt{T}$ \\
\hline
\end{tabular}}
\caption{Comparison of regret and violation bounds under path length ($P_T$) and variation ($V_T$) measures. The “Regularity used” column specifies whether the algorithm requires prior knowledge of path length, variation, both, or neither. We consider three settings: static unconstrained, dynamic unconstrained, and dynamic constrained MABs. In the unconstrained case, our algorithm attains the same regret rate as \citet{chen2025bridging}. To the best of our knowledge, this work establishes the tightest known regret and violation guarantees for non-stationary MABs across both constrained and unconstrained settings. Furthermore, the rates are order-optimal, matching the lower bound in Proposition~\ref{theorem:lowerbound}. When regularity assumptions are removed, we also provide the best known regret and violation bounds in the constrained setting. }
\end{table*}

\section{Related Work} \label{s:related_work}
\paragraph{Online Convex Optimization and Dynamic Comparators.} The full-information counterpart to the MAB problem is the expert problem~\citep{littlestone94}. In the expert problem, a decision-maker aims to select an expert from a pool of options, but unlike the MAB setting, the costs associated with all experts are revealed, even for those not selected. This problem is subsumed by the more general framework of Online Convex Optimization (OCO), which considers a full-information online setting where reward functions are convex (rather than linear) and decisions are chosen from a compact convex set (instead of the probability simplex).  First introduced by~\citet{zinkevich2003}, who showed that projected gradient descent attains sublinear regret. OCO generalizes several online problems, and has exerted a profound influence on the machine learning community~\citep{hazan2016introduction, shalev2012online,mcmahan2017survey}. A variant of OCO closely related to our work focuses on dynamic regret~\citep{zinkevich2003, besbes2015non, jadbabaie2015online, mokhtari2016online}, where performance is evaluated against an optimal time-varying sequence rather than a static optimal decision, wherein the regret is evaluated w.r.t.~an optimal comparator sequence instead of an optimal static decision in hindsight. 

Deriving worst-case bounds for dynamic regret without additional assumptions is unattainable~\citep{jadbabaie2015online}. Therefore, following established methodologies~\citep{jadbabaie2015online}, we impose regularity conditions on the comparator sequence, such as assuming slow growth in its \emph{path length} or \emph{temporal variability} of costs, to derive meaningful performance guarantees.

\paragraph{Non-Stationary MAB.} The seminal work of \citet{auer2002nonstochastic}   introduced the concept of modeling non-stationarity in the costs of MABs as an adversarial process. A significant contribution was the proposal of measuring regret against an arbitrary baseline policy, characterized by a difficulty metric based on the frequency of arm switches. This concept was subsequently formalized as switching regret in later studies~\citep{yu2009piecewise,garivier2011upper}, equivalent to the comparator sequence's path length in the present work. While other research has focused on regret relative to the absolute variation of arm costs~\citep{slivkins2008adapting, besbes2014stochastic}, aligning with the temporal variability inherent to our treatment of the problem. This work adheres to the OCO framework for consistency in terminology regarding path length and temporal variability regularity conditions.

\paragraph{Non-Stationary MAB with Constraints.}
The difficulty of non-stationary MABs is substantially amplified in realistic deployments, where decision-makers must cope not only with fluctuating costs but also with time-varying soft constraints. While several studies~\citep{shi2022bayesian,zhou2022kernelized,deng2022interference} address this setting under static regret, to the best of our knowledge only~\citet{deng2022interference} proposes algorithmic approaches for the more challenging dynamic regret regime with jointly non-stationary costs and constraints. Their method models both costs and constraints using Gaussian processes (GPs) and employs an Upper Confidence Bound (UCB)-style exploration rule; non-stationarity is handled via forgetting mechanisms such as sliding windows and restarts~\citep{cheung2019learning,russac2019weighted,zhou2021no,deng2022weighted}, which effectively construct time-local GP posteriors and UCB statistics. A key limitation of forgetting-based approaches is that performance hinges on committing to a particular forgetting schedule (e.g., a window length or discount factor), often tuned using a known or assumed rate of temporal change and fixed once the algorithm begins; consequently, they can degrade when variation is irregular, adversarial, or unknown. In contrast, our work develops an algorithm with bandit feedback that adapts to non-stationarity \emph{without} requiring such prior tuning, and our analysis provides order-wise optimal dynamic regret and cumulative constraint violation guarantees under standard measures of temporal complexity. A close alternative, \citet{cao2018online}, captures only comparator-drift-type complexity (cf.~\citet{zinkevich2003}), remains within a Euclidean framework with two-point feedback, and uses coupled loss and constraint learning rates that yield feasibility bounds scaling with the amount of drift; by contrast, our \emph{single-point}, \emph{non-Euclidean} primal-dual updates decouple the two effects and provide guarantees that also accommodate changing environments.

\section{Problem Statement} \label{s:problem_formulation}

We consider a system operating over $T \in \naturals$ time slots. In each slot, a decision-maker selects an action $a_t$ from a set of $n  \in \naturals \setminus \set{1}$ possible actions $\A=\{1,2,\ldots,n\}$. Subsequently, the environment discloses the corresponding cost $f_{t, a_t}$ and constraint $g_{t, a_t}$ for the selected action $a_t$. This setting represents a multi-armed bandit problem with bandit feedback.  An oblivious adversary determines the losses and constraints by selecting loss vectors  $\vec f_t\triangleq(f_{t,a})_{a\in A}$ and constraint vectors $\vec g_t\triangleq(g_{t,a})_{a\in A}$ for each time slot $t\in\intv{T}$ at the beginning of the game.  We assume bounded losses and constraints. Formally,
\begin{assumption}\label{asm:boundedness}
    The loss vectors and constraint vectors are generated by an oblivious adversary and are assumed to be uniformly bounded, with 
    $\vec{f}_t \in [0,1]^n$ and $\vec{g}_t \in [-1,1]^n$ for all $t \in \intv{T}$.
\end{assumption}
This boundedness assumption is necessary to obtain meaningful guarantees: if losses were unbounded, an adversary could force an arbitrarily large loss in the first round. Once boundedness is imposed, we normalize both the losses and the constraints solely for notational convenience; \emph{this entails no loss of generality}.

At each time step $t$, the decision-maker selects an action according to a distribution $\x_t$ over the action space $\A$. This distribution is represented by a probability vector $\x_t$ belonging to the $n$-dimensional simplex:
\begin{align}
    \simplex_n \triangleq \set{\x \in [0,1]^n: \norm{\x}_1 = 1}.
\end{align}
The expected loss and constraint violation are defined as $f_t(\x) \triangleq  \vec f_t \,\cdot\, \x$ and $g_t(\x) \triangleq \vec g_t \,\cdot\, \x$ for every  $\x \in \simplex_n$ and $t \in \intv{T}$. For notational convenience, we adopt the shorthand $f_t(a)$ and $g_t(a)$ to denote $f_t(\vec e_a)=f_{t,a}$ and $g_t(\vec e_a)=g_{t,a}$, respectively, where $\vec e_a$ is the unit basis vector with a 1 in the $a$-th position.

We assume the existence of strictly feasible decisions, a condition commonly referred to as Slater’s condition in the literature~\citep{boyd2004convex}. This assumption is standard in the context of time-varying constraints~\citep{valls2020online, deng2022interference}. Formally, we posit the following:
\begin{assumption}\label{asm:slaters}
   There exists a point $\x \in \simplex_{n}$ and a positive constant $1 \geq \rho > 0$ such that $g_t(\x) \leq -\rho$ for $\forall t \in \intv{T}$.
\end{assumption}

\paragraph{Policies.} A \emph{policy} $\vec{\mathcal{P}}$ is the rule for selecting actions based on the history of previously chosen actions together with the corresponding realized costs and constraint values. Formally, we define a policy as a sequence of randomized mappings
\begin{align}
    a_{t+1} = \mathcal{P}_t \left( \big\{ a_s, f_s(a_s), g_s(a_s) \big\}_{s=1}^t \right),
\end{align}
where each $\mathcal{P}_t$ maps the trajectory up to round $t$ to a distribution over feasible actions and samples an action $a_{t+1}$ from it. The initial action is drawn according to an initial distribution $ a_1   = \mathcal{P}_1(\varnothing) \sim \x_1$.

\paragraph{Optimal Comparator Sequence.}  
We select the optimal comparator sequence after $T$ rounds to minimize cumulative loss while adhering to time-varying constraints. For each timeslot $t\in\intv{T}$, we define the feasible subset $\simplex_{n, t} \subseteq \simplex_n$ as the set of feasible points within the simplex:   $\simplex_{n, t} \triangleq \set{\x \in \simplex_n: g_t(\x) \leq 0}$. We aim to determine the best comparator sequence, which represents a solution to the following problem:
\begin{align}
    \set{\x^\star_t }^T_{t=1}  \in \underset{\set{\x_t}^T_{t=1}\in \bigtimes^T_{t =1} \simplex_{n,t}}{\argmin}\set{\sum^T_{t=1} f_t(\x_t)  },\label{eq:comparator}  
\end{align}
\textbf{Performance Metrics.} The goal of a policy is to minimize the cumulative loss relative to the comparator decisions while satisfying the constraints. The regret of the policy and the constraint violations are defined as follows, respectively:
\begin{align}
    \mathfrak{R}_T\parentheses{\vec{\mathcal{P}}} &\triangleq \E_{\vec {\mathcal{P}}}\interval{\sum^T_{t=1} f_t(a_t)} -\sum^T_{t=1} f_t(\x^\star_t), \\
\mathfrak{V}_T\parentheses{\vec{\mathcal{P}}}&\triangleq\E_{\vec {\mathcal{P}}}\interval{\sum^T_{t=1} g_t(a_t)}.
\end{align}
The policy  $\vec {\mathcal{P}}$ aims to achieve both sublinear regret $\mathfrak{R}_T\parentheses{\vec{\mathcal{P}}} = o(T)$ and sublinear violation $\mathfrak{V}_T\parentheses{\vec{\mathcal{P}}} = o(T)$. This implies that the policy's performance gradually approaches that of the best comparator sequence, while remaining feasible on average over time.
\paragraph{Regularity Assumptions.} As demonstrated by~\citet{jadbabaie2015online}, deriving worst-case bounds for dynamic regret without imposing additional assumptions is infeasible.   Consequently, we introduce regularity conditions on the non-stationarity of the problem. 

We employ two distinct metrics to quantify non-stationarity: the path length, denoted by $P_T\in \reals_{\geq 0}$, which indirectly measures fluctuations in the cost functions through the comparator sequence $\set{\x^\star_t}^T_{t=1}$,  and the cost function temporal variation, denoted by $V_T \in \reals_{\geq 0}$, which directly quantifies changes in the cost functions themselves. We formally define these measures as follows:
\begin{align}
    P_T &\triangleq \sum^T_{t=1} \norm{\x^\star_t - \x^\star_{t+1}}_1,&\label{eq:path length} \text{(path length)}\\
    V_T &\triangleq \sum^T_{t=1} \norm{\vec f_t - \vec f_{t+1}}_\infty, \label{eq:function-variation} &\text{(temporal variation)}
\end{align}
and we adopt the following assumption.

\begin{assumption}\label{asm:regularity}
The policy $\vec{\mathcal{P}}$ has access to the regularity measures $P_T$ and $V_T$.
\end{assumption}
This assumption makes explicit the side information used to tune the algorithm. In Section~\ref{s:meta} we show how to remove this requirement.

\paragraph{Incomparability of Regularity Assumptions.} The two non-stationarity measures, path length $P_T$ and temporal variation $V_T$, are complementary, each capturing a different aspect of the non-stationarity of the cost functions.  Path length quantifies the cumulative variability of the cost functions along a comparator sequence, whereas temporal variation directly measures fluctuations of the cost functions themselves. 
These measures are inherently incomparable: there exist cases where $V_T$ grows linearly with time ($V_T = \Omega(T)$) while $P_T$ remains constant ($P_T = \BigO{1}$), as well as cases where $V_T$ is bounded ($V_T = \BigO{1}$) but $P_T$ grows linearly ($P_T = \Omega(T)$).  Concrete examples are provided in Appendix~\ref{app:incomparability}
\paragraph{Fundamental Limits.} Under this problem, the constrained setting strictly generalizes the unconstrained one. 
Therefore, the known lower bound for unconstrained dynamic regret directly transfers to this setting:

\begin{proposition}(\citep{chen2025bridging}, \citep{besbes2015non})\label{theorem:lowerbound}
For given $T, n, V_T, P_T$  satisfying $T \ge n \ge 2$ and $V_T \le T/n$, 
the regret for any policy $\vec{\mathcal{P}}$ is lower bounded as
\begin{align}
\mathfrak{R}\left(\vec{\mathcal{P}}\right) 
= \Omega\left(\min\left\{ V_T^{1/3} T^{2/3},  \sqrt{P_T T} \right\}\right).
\end{align}
\end{proposition}
The $\Omega\left(\sqrt{P_T T}\right)$ term is established by \citep{chen2025bridging}, 
while the $\Omega\left((V_T)^{1/3} T^{2/3}\right)$ term follows from the analysis of 
\citep{besbes2015non}. Combining the two yields the stated lower bound.

\section{\texttt{BCOMD} Policy}\label{s:main_result}
Our policy incorporates a Lagrangian function defined as $  \Psi (\x , \lambda) \triangleq f_t(\x) + \lambda g_t(\x)$. The first term minimizes the cost function, while the second term penalizes soft constraint violations through the Lagrangian multiplier, which serves as an importance weight. Operating in a challenging bandit setting, we estimate the gradient of this function with respect to the distribution $\x \in \simplex_n$, but compute the gradient with respect to the dual variable $\lambda$ exactly.

\subsection{Gradient Estimators} 
In the bandit feedback setting, we construct unbiased estimates of the gradients of the expected loss function $f_t(\x)$ and constraint $g_t(\x)$ using the following estimators:
\begin{align}
  \tilde{\vec f}_t = {f(a_t)}{x^{-1}_{t,a_t}} \vec{e}_{a_t},\qquad 
  \tilde{\vec g}_t = {g(a_t)}{x^{-1}_{t,a_t}} \vec{e}_{a_t}.\label{eq:estimator}
\end{align}
Here, $a_t$ denotes the action selected at time step $t$, and $\vec e_{a_t}$ represents a unit basis vector aligned with action $a_t$. The gradient estimates are unbiased, i.e.:
\begin{align}
    \E_{a \sim \x} \interval{{f(a_t)}  x_{a}^{-1} \vec{e}_{a}} = \nabla f_t(\x),
\end{align}
The same property holds for the constraint gradient estimates~$\tilde{\vec g}_t$.

Online Gradient Descent (OGD) methods are standard tools in online learning~\citep{hazan2016introduction,shalev2012online,mcmahan2017survey}. However, they are not applicable in our setting: their $\mathcal{O}(\sqrt{T})$ regret guarantees hinge on a bounded second moment of the (possibly stochastic) gradient estimator. Under bandit feedback over the simplex this condition fails—the second moment is unbounded,
\begin{align}
    \sup\set{\E_{a\sim \x}\interval{\norm{\tilde{\vec f}_t}_2^2}: {\x\in\simplex_n}} = \infty
\end{align}
whenever $f_t(a)>0$ for some a, since taking $x_a \to 0$ makes the estimator explode. In this regime, the Euclidean-norm variance terms that appear in OGD analyses~\citep{hazan2016introduction} can diverge in the bandit regime,  and consequently no meaningful regret bounds can be established.  While two-point (or multi-point) bandit schemes can control this variance~\citep{chen2018bandit}, such feedback is unavailable in our setting.


To address this challenge, we propose leveraging the online mirror descent (OMD) framework. OMD allows us to employ the same gradient estimates while departing from Euclidean geometry. A notable example is the EXP-3 policy for adversarial bandits under static regret and no constraints setting~\citep{hazan2014bandit} constitutes a specific instance of OMD that employs an entropic setup and the gradient estimate defined in Equation~\eqref{eq:estimator}, achieving sublinear regret bounds.

\begin{algorithm}[t]
 \caption{\texttt{BCOMD}: Bandit-Feedback Mirror Descent with
Time-Varying Constraints.\label{alg:omd}}
\begin{scriptsize}
    \begin{algorithmic}[1]
        \Require Initial distribution $\vec{x}_1 = \vec{1}/n$, Initial dual variable $\lambda_1 = 0$, Mirror map $\Phi: \mathcal{D}\subseteq \mathbb{R}^n \to \mathbb{R}$, Learning rate $\eta \in \mathbb{R}_{>0}$, Shift parameter $\gamma \in [0,1/n]$
        \State Assign $\Omega$ as in Eq.~\eqref{eq:bounded_dual_omega_main}.
        \For {$t = 1, \dots, T$}
        \State Sample action $a_t \sim \vec{x}_t$
        \State Incur $f_t(a_t)$ and $g_t(a_t)$ \Comment{Bandit-feedback costs and constraints}
        \State Construct estimates $\tilde{\vec{f}}_t$ and $\tilde{\vec{g}}_t$~\eqref{eq:estimator}
        \State $\tilde{\vec{\omega}}_t \gets \tfrac{\Omega}{x_{t,a_t}}\vec{e}_{a_t}$ 
        \Comment{Stabilizer to ensure non-negative pseudo-costs}
        \State $\tilde{\vec{b}}_t \gets \tilde{\vec{\omega}}_t + \tilde{\vec{f}}_t + \lambda_t \tilde{\vec{g}}_t$ 
        \Comment{Gradient estimate of $\Psi(\cdot,\lambda_t)$ at $\vec{x}_t$}
        \State $\vec{y}_{t+1} \gets (\nabla \Phi)^{-1}\bigl(\nabla \Phi(\vec{x}_t) - \eta \tilde{\vec{b}}_t \bigr)$ 
        \Comment{Adapt primal variable}
        \State $\vec{x}_{t+1} \gets \Pi^\Phi_{\simplex_{n,\gamma} \cap \mathcal{D}}\bigl(\vec{y}_{t+1}\bigr)$
        \Comment{Bregman projection onto $\simplex_{n,\gamma}$}
        \State $\lambda_{t+1} \gets \bigl(\lambda_t + \mu g_t(a_t)\bigr)_+$ 
        \Comment{Adapt dual variable} 
        \EndFor 
    \end{algorithmic}
\end{scriptsize}
\end{algorithm}
\setlength{\leftmargini}{0.75em}   
\setlength{\labelsep}{0.5em}    
\subsection{Mirror Descent Policies}
We propose a policy that employs OMD framework to address the challenge of controlling estimator variance. The core principle of mirror descent is the distinction between two spaces: the primal space for variables and the dual space for supergradients. A mirror map connects these spaces. Unlike standard online gradient descent, updates are computed in the dual space before being mapped back to the primal space using the mirror map. For several constrained optimization problems of interest, OMD exhibits faster convergence compared to OGD~\cite[Section~4.3]{bubeck2011introduction}. Within our specific context, we will establish that appropriately configured OMD yields tighter regret bounds with gradient estimators where OGD fails.  The OMD policy is well-defined under the following assumption.
\begin{assumption}\label{asm:mirror_map}
Let $\Phi:\D\to\reals$ be a mirror map, where $\D\subset\reals^n$ is an open, convex set. We assume that $\Phi$ is compatible with the simplex geometry and satisfies the following standard regularity conditions:
\begin{enumerate}
\item The probability simplex $\simplex_n$ is contained in the closure of the domain, i.e., $\simplex_n\subset\mathrm{closure}(\D)$, and it has a nonempty intersection with the interior domain, $\simplex_n\cap\D\neq\varnothing$.
\item The function $\Phi$ is strictly convex and continuously differentiable on $\D$, namely $\Phi\in C^1(\D)$.
\item The gradient mapping induced by $\Phi$ covers the entire dual space, i.e., $\nabla\Phi(\D)=\reals^n$.
\item  The gradient norm diverges as one approaches the boundary of the domain:
$|\nabla\Phi(x)|\to\infty$ as $x\to\partial\D$.
\end{enumerate}
\end{assumption}

A map $\Phi$ satisfying these properties is termed a \emph{mirror map}. In particular, the above assumption implies the existence of the inverse of the mapping induced by gradient $\nabla \Phi$, and the uniqueness of the Bregman projection step in Line 9 Algorithm~\ref{alg:omd}. Formally, defined as 

\begin{definition}
    Let $\Phi: \D \to \reals$  be a mirror map. 
The Bregman projection associated to a map $\Phi$ onto a convex set $\mathcal S$ is  denoted by $\Pi^{\Phi}_{\mathcal{S}}:\mathbb{R}^n\to\mathcal{S}$, is defined as
    \begin{align}
        \Pi^{\Phi}_{\mathcal{S}}(\y') &\triangleq \underset{\y \in {\mathcal{S}}}{\arg\min}\,D_\Phi (\y,\y') \label{eq:Bregman_projection},
    \end{align}
    where $D_\Phi(\y,\y') \triangleq \Phi(\y) - \Phi(\y') - \nabla {\Phi(\y')} \cdot (\y - \y')$ is the Bregman divergence $D_\Phi: \D\times \D \to \reals_{\geq 0 }$ associated to the map $\Phi$.
\end{definition}

Our policy, detailed in Algorithm~\ref{alg:omd}, iteratively updates the action distribution. This update is informed by an OMD step that incorporates a gradient estimate of the cost function and a weighted gradient estimate of the constraint function (Lines 8--9, Algorithm~\ref{alg:omd}). Crucially, we dynamically adjust the weights applied to the constraint function's gradients based on a cumulative measure of constraint violation (Line 10, Algorithm~\ref{alg:omd}).

\subsection{Entropic Instantiation}  We derive our primary results for the entropic OMD algorithm. This specific instance employs the negative entropy mirror map $\Phi: \D \to \mathbb{R}$ defined as
\begin{align}
        \Phi(\x) \triangleq \sum^n_{i =1} x_i \log(x_i) \qquad \text{for $\x \in \D\triangleq \reals^n_{> 0}$}.\label{eq:neg_entropy_mirrormap}
\end{align}
It is readily verified that this mapping satisfies Assumption~\ref{asm:mirror_map}.

We define the algorithm's domain as the restricted simplex  $\simplex_{n, \gamma} = \simplex_n \cap [\gamma, 1]^n$.  This construction ensures that the probability of selecting any action remains above a predetermined threshold, $\gamma$. This parameter is crucial for establishing the algorithm's theoretical guarantees. Intuitively, $\gamma$ controls the algorithm's exploration, which is essential for adapting to distribution shifts and rapidly identifying optimal actions.
We present our primary regret bound achieved by Algorithm~\ref{alg:omd} when configured with entropic OMD. Formally,

\begin{theorem}\label{theorem:main}
    Let $c_T  = {\min \set{\sqrt{P_T}, V_T^{1/3} T^{1/6}}}  $, $\mu = {{1}/({M \sqrt{T}})}$, $\eta = \max \set{1, c_T} / (M\sqrt{T})$,  $\gamma = \Theta({T^{-1/2}})$,  and $M = 4\parentheses{\frac{3n+2}{\rho} + 1}^2$. Under Assumptions~\ref{asm:boundedness}, \ref{asm:slaters}, and \ref{asm:regularity}, the policy $\vec{\mathcal{P}}$ produced by Algorithm~\ref{alg:omd} satisfies
    \begin{align}
       \mathfrak{R}_T\parentheses{\textup{\texttt{BCOMD}}} &= \tilde{\mathcal{O}}\parentheses{\min\set{\sqrt{P_T T}, V_T^{1/3} T^{2/3}}},\nonumber\\
        \mathfrak{V}_T\parentheses{\textup{\texttt{BCOMD}}} &= \tilde{\mathcal{O}}\parentheses{\sqrt{T}}.
    \end{align}
\end{theorem}
We employ the operator $\tilde{\mathcal{O}}\parentheses{\,\cdot\,}$ as a variant of the standard big-O notation, $\BigO{\,\cdot\,}$, where logarithmic factors are omitted. A sketch of the proof is provided in the following section. The full proof is deferred to the Appendix.

When decision-makers possess limited knowledge of optimal path length ($P_T$) and temporal variation ($V_T$), we propose a meta-algorithm designed to learn these parameters in Section~\ref{s:meta}.

\section{Formal Analysis of Performance Guarantee}
We start by presenting various technical lemmas that support the main proof.
\subsection{Supporting Lemmas}
    The history $\mathcal{H}_t$ of Algorithm~\ref{alg:omd} is fully characterized by its random actions and received feedback up to time $t$ given by:
    \begin{align}
        \mathcal{H}_t \triangleq  \set{\parentheses{ a_s, f_s(a_s), g_s(a_s)} }_{s \in \intv{t}}.
    \end{align}

\begin{lemma}\label{lemma:bounded_Bregman}
 Under the negative entropy mirror map~\eqref{eq:neg_entropy_mirrormap}, the primal iterates of Algorithm~\ref{alg:omd} satisfy
 \begin{align}
\E\interval{D_\Phi(\x_t, \y_{t+1}) \Big|\hist_{t-1}} &\leq    {1.5\eta^2 n} (1 +  \lambda_t^2 +  \Omega^2).\nonumber
\end{align}
\end{lemma}
\begin{proof}[Proof sketch]
We obtain the lemma by instantiating the Bregman divergence induced by the negative-entropy mirror map. 
Using an appropriate upper bound, we show that the expected discrepancy between the iterates 
$\x_t$ and $\y_{t+1}$, measured in this divergence, remains bounded. 
This guarantee is specific to the mirror-descent geometry and does not extend to OGD, where the analogous term 
reduces to the Euclidean norm of the (estimated) gradients and need not be bounded.
\end{proof}
The full proof is deferred to Appendix~\ref{app:bounded_bregman}.
\begin{lemma}\label{lemma:saddle_inequality}
 Under the negative entropy mirror map~\eqref{eq:neg_entropy_mirrormap}, for any $t \in \intv{T}$ the variables $\lambda_t$ and $\x_t$ in Algorithm~\ref{alg:omd} satisfy the following inequality:
    \begin{align}
       &\E \interval{    \Psi_t(\x_t, \lambda) - \Psi_t(\x, \lambda_t)} \\
       &\leq \E\interval{\eta^{-1}\parentheses{ D_{\Phi}(\x, \x_t) - D_{\Phi}(\x, \x_{t+1})}}\nonumber\\
       &+ \E\interval{\parentheses{2\mu}^{-1}\parentheses{\parentheses{\lambda_t - \lambda}^2 - \parentheses{\lambda_{t+1} - \lambda}^2}}\nonumber \\
        &+ \E\interval{ \parentheses{{\eta} ^{-1}{D_\Phi(\x_t, \y_{t+1})} + {\mu /2}}},~\text{for $\x \in \simplex_n$, $\lambda \geq 0$}.\nonumber
    \end{align}
\end{lemma}
\begin{proof}[Proof sketch]
We derive the following lemma by exploiting the linearity of the function $\Psi$ with respect to its arguments $\x$ and $\lambda$. Since both iterates produced by Algorithm~\ref{alg:omd} are driven by gradient algorithms, with carfeul analysis we can bound the above difference with respect to the update rule of OMD over the distributions and standard OGD over the constraint weights.
\end{proof}
The full proof is deferred to Appendix~\ref{app:saddle_inequality}.

\begin{lemma}\label{eq:bounded_dual}
 Under the negative entropy mirror map~\eqref{eq:neg_entropy_mirrormap}, the dual variables $\lambda_t$ for $t \in \intv{T}$ in Algorithm~\ref{alg:omd}  are bounded by 
    \begin{align}
        \Omega \triangleq  \frac{\log\parentheses{\gamma^{-1}}}{\rho} \parentheses{\frac{\mu}{\eta}}  +\frac{3  n }{2\rho} \eta + \frac{1}{2\rho} \mu  + \frac{3  n }{\rho} + \frac{2}{\rho}  +  1,\label{eq:bounded_dual_omega_main}
    \end{align}
    for $\eta \leq \Omega^{-2}$ and $\mu \in  \reals_{>0}$. 
\end{lemma}
\begin{proof}[Proof sketch]
    We establish the following Lemma by contradiction. Assuming the existence of a $t=t_0$ such that $\lambda_t > \Omega$, we demonstrate, in conjunction with Lemma~\ref{lemma:saddle_inequality}, that this assumption is untenable.
\end{proof}
The full proof is deferred to Appendix~\ref{app:bounded_dual}.

\begin{lemma} 
Let $\gamma \in \reals_{> 0}$ and $\set{\x^\star_{t, \gamma}}^T_{t=1}$ be the projection of the comparator sequence~\eqref{eq:comparator} onto $\simplex_{n,\gamma}$.  Under the negative entropy mirror map~\eqref{eq:neg_entropy_mirrormap}, the primal decisions $\set{\x_t}^T_{t=1}$ of Algorithm~\ref{alg:omd} satisfy the following:
    \begin{align}
       &\sum^T_{t=1}\parentheses{D_\Phi(\x^\star_{t,\gamma}, \x_t) - D_\Phi(\x^\star_{t,\gamma}, \x_{t+1})}  \\
       &\leq \log(1/\gamma)\parentheses{1 + 2 \sum^{T-1}_{t=1} \norm{ \x^\star_{t+1} - \x^\star_t}}.
    \end{align}
\end{lemma}
\begin{proof}[Proof sketch]
We extend the standard mirror descent analysis of \citet{bubeck2011introduction} to establish the above relationship. Specifically, for fixed comparator sequences, we derive telescoping terms upper-bounded by the range term $\log(n)$, coinciding with the diameter of our set under the Bregman divergence associated with the negative entropy mirror map. In contrast to the classical approach, our analysis for non-fixed comparator sequences highlights the indispensable role of the parameter $\gamma>0$ in deriving regret guarantees within the dynamic setting.
\end{proof}
The full proof is deferred to Appendix~\ref{app:path}.

\begin{lemma}\label{lemma:n0}
    Consider a comparator sequence $\set{\vec u_t}^T_{t=1} \in \simplex_n^T$. Under the negative entropy mirror map~\eqref{eq:neg_entropy_mirrormap}, Algorithm~\ref{alg:omd} has the following regret guarantee under Asms.~\ref{asm:boundedness} and~\ref{asm:slaters}   for $\beta = n (3/2 + 3 \Omega^2) $, $\eta = \frac{\eta_0}{\sqrt{\beta T}}$ and $\eta_0 > \log(n)/\sqrt{2} $:
    \begin{align}
         &\E\interval{ \sum^T_{t=1} f_t(a_t) - f_t(\vec u_{t}) }  \\
         &\leq \frac{\parentheses{\log(n) + 2 \log(1/\gamma) \sum^{T-1}_{t=1}\norm{\vec u_t- \vec u_{t+1}}_1 }  }{\eta_0}\sqrt{\beta T}  \nonumber\\
         &+ \eta_0 \sqrt{\beta T}+ 2\gamma  T + \frac{\mu T}{2}.
    \end{align}
\end{lemma}
\begin{proof}[Proof sketch]
We employ Lemmas 1 through 4 for this proof. Specifically, summing Lemma 2 from $t = 1$ to $T$ and setting $\lambda = 0$ completes the argument.
\end{proof}
The full proof is deferred to Appendix~\ref{app:path_length}.

\begin{lemma}     Assume that $\lambda_t \le \Omega$ for all $t \in [T]$. Fix an arbitrary comparator sequence $\{\vec u_t\}_{t=1}^T \in \simplex_n^T$. We run Algorithm~\ref{alg:omd} with the negative-entropy mirror map in~\eqref{eq:neg_entropy_mirrormap}. Let $\beta = n(3/2 + 3\Omega^2)$ and set $\eta = \eta_0/\sqrt{\beta T}$ with $\eta_0 > \log(n)/\sqrt{2}$. Under these choices, we obtain the following regret guarantee. \label{lemma:n02}
    \begin{align}
     &\E\interval{ \sum^T_{t=1} f_t(a_t) - f_t(\vec u_{t}) }  \leq 2 \eta_0 \sqrt{\beta T}+ 2\gamma  T + \frac{\mu T}{2} \nonumber\\
     &+ \frac{4 \log(1/\gamma) T V_T }{\eta_0^2 - \log(n)/2} \mathds{1}\parentheses{E_T}.
    \end{align}
    where $E_T \triangleq \parentheses{\sum^{T-1}_{t=1}\norm{\vec u_t- \vec u_{t+1}}_1 > \frac{ \eta_0 ^2 - \log(n)/2}{\log(1/\gamma)}}$.
\end{lemma}
\begin{proof}[Proof sketch]
We extend the argument of Lemma~2 in \citet{jadbabaie2015online} to complete the proof.
\end{proof}
The full proof is deferred to Appendix~\ref{app:temporal_variation}.
\subsection{Proof of Theorem~\ref{theorem:main}}
\label{proof:main}
\begin{proof}
      Combining Lemmas~\ref{lemma:n0} and~\ref{lemma:n02}, we obtain
    \begin{align}
          &\E\interval{ \sum^T_{t=1} f_t(a_t) - f_t(\vec u^\star_{t}) }\\
          &= \min \Big\{ 2 \eta_0 \sqrt{\beta T} + \frac{4 \log(1/\gamma) T V_T }{\eta_0^2 - \log(n)/2} , \\
          &\parentheses{ \frac{\parentheses{\log(n)/2 + \log(\frac{1}{\gamma}) P_T  }  }{\eta_0}   +  \eta_0 }\sqrt{\beta T}\Big\} +2 \gamma T + \frac{\mu T}{2}.
    \end{align}
    We next set $\eta = \eta_0/\sqrt{\beta T}$ with
$\eta_0 = \min\{\sqrt{P_T}, V_T^{1/3}T^{1/6}\}/M$, choose $\mu = 1/(M\sqrt{T})$, and take $\gamma=\Theta(T^{-1/2})$.
Substituting these parameter choices into the preceding bound yields
    \begin{align}
        \E\interval{ \sum^T_{t=1} f_t(a_t) - f_t(\vec u^\star_{t}) } = \BigO{\min\set{\sqrt{P_T T},{V_T^{1/3} T^{2/3}}}}.
    \end{align}
For the cumulative constraint violation, note that the dual update implies
\begin{align}
\mu^{-1}\E\interval{ \lambda_{t+1} -  \lambda_t\Big| \hist_{t-1}}  &\geq  g_t(\x_t).
\end{align}
Therefore, for $\mu=\Theta(T^{-1/2})$, we have
\begin{align}
    &\E\interval{\sum^T_{t=1} g_t(a_t)} = \E \interval{\sum^T_{t=1} g_t(\x_t)}\\
    &\leq \E\interval{\sum^T_{t=1} {\mu } ^{-1}\E\interval{ \lambda_{t+1} -  \lambda_t\Big| \hist_{t-1}}}   \\
    &= {\mu }^{-1}{\E[\lambda_{T+1}]}  \leq {\mu}^{-1}{\Omega} = \BigO{\sqrt{T}} \label{eq:bound_proof_dual},
\end{align}
where the penultimate inequality uses $\E[\lambda_{T+1}]\le \Omega$.

Finally, our choice of $M$ ensures that $\eta \le \Omega^{-2}$ (Appendix~\ref{app:bounded_M}); hence, Lemma~\ref{eq:bounded_dual} guarantees that the dual iterates remain bounded by $\Omega$, which justifies the bound in~\eqref{eq:bound_proof_dual}. This completes the proof.
\end{proof}

\subsection{Meta-Algorithm for Learning Path Length and Temporal Variation}\label{s:meta}

We eliminate the need for prior knowledge of $P_T$ or $V_T$ by employing a 
\emph{meta-algorithm} (\texttt{MBCOMD} in Algorithm~\ref{alg:meta-doubling} in the Appendix). The time horizon is partitioned into 
geometrically growing phases $\mathcal I_m$ of length $L_m = 2^{m-1}$, 
for $m = 1, \dots, M = \lceil \log_2 T \rceil$.  In each phase $\mathcal I_m$, we run $K_m = \lceil \log_2 L_m \rceil$ instances of
\texttt{BCOMD} from a geometric grid. At the beginning of every phase, the dual variable 
is reset.  The meta-learner maintains a mixture distribution over these experts and 
aggregates their predictions. Bandit feedback is broadcast to all experts, 
enabling them to update via entropic mirror descent with appropriate expert-level 
losses (see Algorithm~\ref{alg:meta-doubling} in the Appendix). 
\paragraph{Bandit with Expert Advice (Meta-Bandits).}
Formally, let $\mathcal{A}$ be the action set as before and $\mathcal{E}$ the expert set with $|\mathcal{E}|=K$. 
At each round $t$: 
(1) each expert $k\in\mathcal{E}$ outputs a distribution $\vec x^{(k)}_t \in \simplex_n$; 
(2) the learner samples and plays action $a_t \sim \xmeta_t \in \simplex_K$; 
(3) the learner observes the bandit feedback for $a_t$ and constructs unbiased estimates of cost and constraint; 
(4) the learner updates $\xmeta_t$ and the internal states of the experts. 
The regret in this setting satisfies  $\widetilde{\mathcal{O}}\left(\sqrt{nL \log K}\right)$ over any block of length $L$ (see Lemma~\ref{lem:meta-phase} in Appendix), i.e., sublinear in the number of actions $n$ and only logarithmic in the number of experts $K$. 
In our setting, $K_m = \Theta(\log L_m)$, so the per-phase overhead is $\mathcal{O}(\sqrt{nL_m \log K_m})$. 
Summing across phases with the doubling schedule gives a global overhead $\mathcal{O}(\sqrt{nT \log\log T})$, which is absorbed into $\widetilde{\mathcal{O}}(\cdot)$.

\begin{theorem}[Adaptive regret and violation via doubling]\label{thm:meta}
Under Assumptions~\ref{asm:boundedness} and \ref{asm:slaters}, \textup{\texttt{MBCOMD}} (Algorithm~\ref{alg:meta-doubling} run across phases $m=1,\dots,\lceil\log_2 T\rceil$) achieves
\begin{align}
\mathfrak{R}_T(\textup{\texttt{MBCOMD}}) 
&= \tilde{\mathcal O}\!\left(\min\set{\sqrt{P_T} T^{2/3},\, V_T^{1/3} T^{7/9}}\right),\nonumber \\
\mathfrak{V}_T(\textup{\texttt{MBCOMD}}) 
&= \tilde{\mathcal O}(\sqrt{T}).
\end{align}
\end{theorem}
The full proof is provided  in the Appendix~\ref{s:appenix_meta}.
\section{Numerical Evaluation}\label{s:experimental_results}
We conclude with a numerical evaluation demonstrating the effectiveness of our proposed policy.
\paragraph{Policies and Hyperparameters.} We implement our proposed policy, \texttt{BCOMD}, in Algorithm~\ref{alg:omd} under the entropic setup. We configure the hyperparameters of \texttt{BCOMD} as follows:  $\eta \in [10^{-3}, 4 \times 10^{-2}]$, $\gamma \in [10^{-5},10^{-3}]$, and $\mu = \eta /2$, and $\Omega = 0$.  Additionally, we implement the \texttt{R-GP-UCB} policy~\citep{deng2022interference} as a benchmark. This policy employs Gaussian Processes (GP) coupled with an Upper Confidence Bound (UCB)-inspired exploration strategy. For \texttt{R-GP-UCB}, we adopt a windowed restarting approach.  For \texttt{R-GP-UCB}, we tune the five hyperparameters $\lambda \in [5 \times 10^{-2}, 10^{-1}]$, $W \in [2\times 10^{3}, 8 \times 10^{3}]$, $\delta \in [10^{-4}, 5\times 10^{-3}]$, $R \in [1,2]$, $\tau \in [10^{-3}, 10^{-2}]$. To ensure fairness, we maintain a consistent resolution across all hyperparameter grids during the search process. We run the experiment on Intel(R) Xeon(R) Platinum 8164 CPU (104 cores). 

\paragraph{Non-Stationary Setup.} We construct a demonstrative example with $n=25$ arms.  A synthetic environment was constructed with fixed cost and constraint functions as initial conditions; each arm is associated with costs $f(a)=1+\sin(\pi(a/(n-1)))$ and $g(a) = 0.5\mathds{1}(a \leq n/1.5) - 1/4$. To introduce non-stationarity, the cost and constraint values were cyclically shifted by five arms within each time window of length $W = 2 \times 10^{2}$, repeating this process six times. The constraints are truncated to values within $[-10^{3}, \infty)$. Gaussian noise $\xi \sim N(0,0.1)$ was added independently to the function values at each timeslot to simulate real-world variability. The trace is depicted in Figure~\ref{fig:traces}.
\begin{figure}[t]
    \centering
    \subcaptionbox{Costs}{\includegraphics[width=0.45\linewidth]{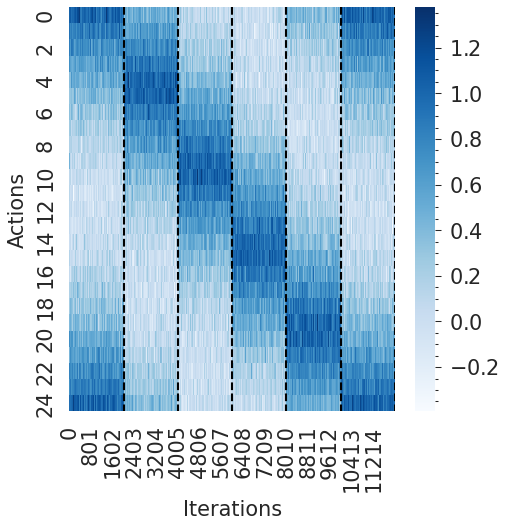}}
    \subcaptionbox{Constraints}{\includegraphics[width=0.45\linewidth]{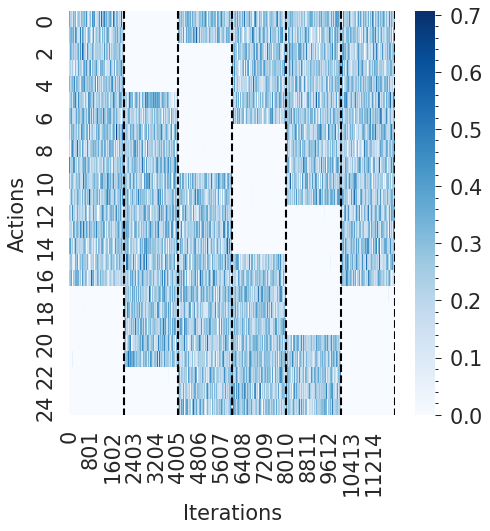}}
    \caption{Non-stationary Trace Generation Setup. The underlying function changes every $2 \times 10^3$ timeslots for six times.}\label{fig:traces}
\end{figure}
\begin{figure}[t]
    \centering
    \subcaptionbox{}{\includegraphics[width=0.45\linewidth]{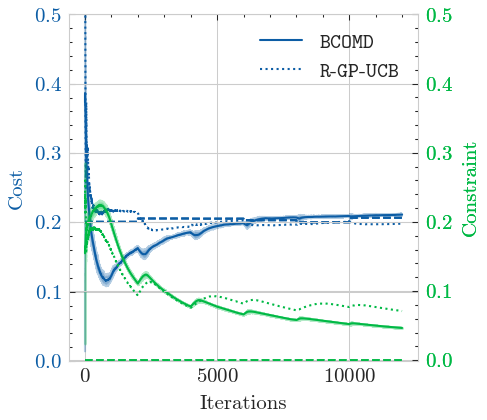}}
    \subcaptionbox{}{\includegraphics[width=0.38\linewidth]{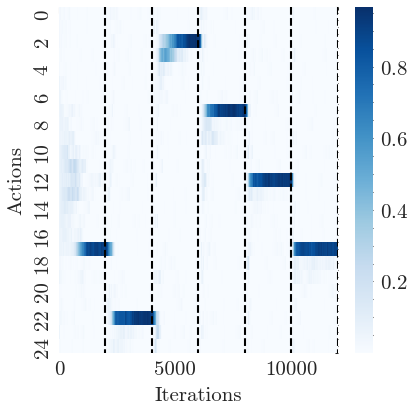}}
    \caption{Subfigure~(a) illustrates the cumulative costs and constraint satisfaction of both \texttt{BCOMD} and \texttt{R-GP-UCB} policies under a non-stationary environment. Subfigure~(b) presents the action distribution $\A$ acquired by the \texttt{BCOMD} algorithm. }
    \label{fig:curve}
\end{figure}
\paragraph{Discussion.} We execute both \texttt{BCOMD} and \texttt{R-GP-UCB} policies on the generated trace. Figure~\ref{fig:curve}~(a) presents the simulation results corresponding to the optimal hyperparameter values identified through grid search. The figure clearly demonstrates the superiority of our approach by achieving concurrently lower cumulative costs and constraint violation. Furthermore, we showcase in Figure~\ref{fig:curve}~(b) the robustness of \texttt{BCOMD} by visualizing the learned distribution over arms across time, which  reveals the policy's ability to identify optimal and feasible actions in diverse time slots.

\section{Conclusion}\label{s:conclusion}
In this work, we present a novel primal-dual algorithm that extends online mirror descent by incorporating tailored gradient estimators and robust constraint handling. We provide rigorous theoretical analysis establishing sublinear dynamic regret and sublinear constraint violation for our proposed policy. Our empirical findings corroborate our theoretical results, demonstrating state-of-the-art performance in terms of both regret and constraint violation.

As avenues for future research, we propose exploring the impact of different mirror maps on algorithm performance. In particular, we are interested in investigating the Tsallis entropy map as a potential alternative to the standard entropy map. This exploration is motivated by the promising results reported in the literature~\citep{zimmert2021tsallis}, where the Tsallis entropy map has been shown to achieve optimal regret bounds in both adversarial and stationary stochastic settings without requiring involved concentration reasoning. Secondly, we intend to assess the adaptability of our proposed methods to a wider range of bandit problem formulations, such as combinatorial bandits and bandit-feedback submodular maximization. By incorporating time-varying constraints, we anticipate addressing a broader spectrum of real-world applications.

\bibliographystyle{apalike}
\bibliography{paperbib}

\input{supp}

\end{document}

%% file: defs.tex
\usepackage{amsmath, amssymb, amsthm, mathtools}
\usepackage{dsfont}
\usepackage{bm}
\usepackage{stmaryrd}   

\usepackage{array}
\usepackage{multirow}
\usepackage{subcaption}
\usepackage{eqparbox,xparse}

\usepackage{xcolor}
\usepackage{pifont}
\usepackage{mdframed}

\newcommand{\simplex}{
  \mathchoice
    {\includegraphics[height=1.4ex, trim= 0pt 30pt 0pt 0pt]{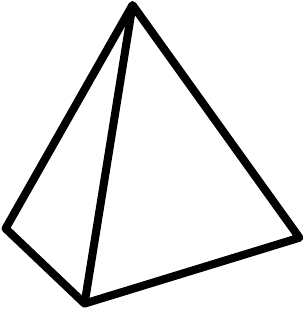}} 
    {\includegraphics[height=1.4ex, trim= 0pt 25pt 0pt 0pt]{figs/simplex.pdf}} 
    {\includegraphics[height=1.1ex, trim=0pt 30pt 0pt 0pt]{figs/simplex.pdf}} 
    {\includegraphics[height=.8ex, trim=0pt 30pt 0pt 0pt]{figs/simplex.pdf}} 
}

\DeclareFontFamily{U}{mathx}{\hyphenchar\font45}
\DeclareFontShape{U}{mathx}{m}{n}{
      <5> <6> <7> <8> <9> <10>
      <10.95> <12> <14.4> <17.28> <20.74> <24.88>
      mathx10
}{}
\DeclareSymbolFont{mathx}{U}{mathx}{m}{n}

\renewcommand{\vec}[1]{\bm{#1}}

\newcommand{\x}{\vec{x}}
\newcommand{\y}{\vec{y}}
\newcommand{\z}{\vec{z}}

\newcommand{\xmeta}{\x^{\mathrm{meta}}}

\newcommand{\X}{\mathcal{X}}

\newcommand{\D}{\mathcal{D}}
\newcommand{\A}{\mathcal{A}}

\newcommand{\hist}{\mathcal{H}}

\newcommand{\reals}{\mathbb{R}}
\newcommand{\naturals}{\mathbb{N}}

\newcommand{\interval}[1]{\left[#1\right]}
\newcommand{\norm}[1]{\lVert#1\rVert}
\newcommand{\abs}[1]{\left|#1\right|}
\newcommand{\set}[1]{\left\{#1\right\}}
\newcommand{\parentheses}[1]{\left(#1\right)}
\newcommand{\inner}[2]{ \parentheses{#1}\cdot\parentheses{#2}}
\newcommand{\intv}[1]{\llbracket #1 \rrbracket}

\newcommand{\card}[1]{\left|#1\right|}

\newcommand{\E}{\mathbb{E}}
\newcommand{\BigO}[1]{\mathcal{O}\parentheses{#1}}

\newcommand{\argmin}{\mathrm{argmin}}

\newcommand{\xmark}{\ding{53}}

\makeatletter
\NewDocumentCommand{\eqmathbox}{o O{c} m}{%
  \IfValueTF{#1}
    {\def\eqmathbox@##1##2{\eqmakebox[#1][#2]{$##1##2$}}}
    {\def\eqmathbox@##1##2{\eqmakebox{$##1##2$}}}
  \mathpalette\eqmathbox@{#3}
}
\makeatother

\renewcommand{\paragraph}[1]{\par\noindent\textbf{#1}}

\makeatletter
\newcommand{\thickhline}{%
    \noalign {\ifnum 0=`}\fi \hrule height 1pt
    \futurelet \reserved@a \@xhline
}
\newcolumntype{"}{@{\hskip\tabcolsep\vrule width 1pt\hskip\tabcolsep}}
\makeatother

\newtheorem{theorem}{Theorem}
\newtheorem{proposition}{Proposition}
\newtheorem{lemma}{Lemma}
\newtheorem{definition}{Definition}
\newtheorem{assumption}{Assumption}


%% file: supp.tex
\clearpage
\onecolumn 
\appendix  

{
  \centering
  
  \hrule height 4pt 
  \vspace{0.15in} 
  
  {\LARGE \bfseries Technical Appendix \par}
  
  \vspace{0.15in}
  
  \hrule height 1pt 
  \vspace{0.3in} 
}

\startcontents[appendices]

This appendix provides supporting results and deferred proofs.
Appendix~A reviews standard preliminaries on Bregman divergences and projections and derives a one-step inequality for Online Mirror Descent.
Appendix~B specializes these tools to the negative-entropy mirror map, establishing KL-diameter and quadratic entropic bounds.
Appendix~C develops problem-specific lemmas for the bandit primal--dual analysis, including control of the conditional divergence term, a saddle-point inequality, and bounds on the dual iterates.
Appendix~D proves the main regret guarantees, including bounds in terms of comparator path length $P_T$ and temporal variation $V_T$.
Appendix~E presents the adaptive meta-algorithm and its analysis.
Finally, Appendix~F discusses computational complexity, and Appendix~G provides examples showing that $P_T$ and $V_T$ are generally incomparable.
\printcontents[appendices]{}{1}{\section*{Contents}}
\clearpage

\section{Standard Preliminaries: Bregman Divergences and Projections}
\subsection{Three-Point (``Pythagorean'') Identity for Bregman Divergences}
\begin{lemma}\label{lemma:Pythagorean}
Let $\Phi:\D\to\reals$ be differentiable and let $D_\Phi$ be the associated Bregman divergence~\eqref{eq:Bregman_projection}. Then for any $\x,\y,\z\in\D$,
    \begin{align}
D_\Phi (\x, \y)+D_\Phi (\z, \x)-D_\Phi (\z, \y) = \parentheses{\nabla \Phi(\x) - \nabla \Phi(\y)} \cdot (\x- \z).
    \end{align}
\end{lemma}
\begin{proof}
Expand the l.h.s. expression via the definition of Bregman divergences in Eq.~\eqref{eq:Bregman_projection}:
\begin{align} 
    &D_\Phi(\x, \y)+D_\Phi(\z, \x)-D_\Phi(\z, \y) \\
    &=  (\Phi(\x)-\Phi(\y)-\nabla \Phi(\y)\cdot (\x-\y)) + (\Phi(\x)-\Phi(\x)-\nabla \Phi(\x) \cdot(\z-\x)) \nonumber\\
    &-(\Phi(\z)-\Phi(\y)-\nabla \Phi(\y) \cdot (\z-\y) \\
    &=  (\nabla \Phi(\x)-\nabla \Phi(\y))\cdot (\x-\z).
\end{align}
We conclude the proof.
\end{proof}
\subsection{Bregman Projection Inequality}
\begin{lemma}\label{lemma:projection_property}
    Fix $\y\in\mathcal{D}$ and let $\x=\Pi^\Phi_{\X}(\y)$ be the Bregman projection of $\y$ onto $\X$. Then for all $\x'\in\X$,
    \begin{align}
        D_\Phi (\x', \y) \geq D_\Phi (\x', \x).
    \end{align}
\end{lemma}
\begin{proof}
First, we show that the following holds
\begin{align}
    \parentheses{\nabla \Phi(\y)-\nabla \Phi(\x)}\cdot \parentheses{\x' - \x } \leq 0\qquad \forall  \x' \in \X.
\end{align}
Since $\x = \Pi^\Phi_{\X} (\y) = \argmin_{\x' \in \X} D_\Phi(\x', \y)$. The first order optimality condition~\cite{bubeck2015convex} gives:
\begin{align}
    \nabla_{\x} D_\Phi (\x,\y) \cdot \parentheses{\x  - \x'} \leq 0,\qquad \forall \x' \in \X.
\end{align}
By taking $\nabla_{\x} D_\Phi (\x,\y) =  \nabla \Phi(\x)-\nabla \Phi(\y)$ completes the proof of the first part. Second, Lemma~\ref{lemma:Pythagorean} gives 
\begin{align}
    D_\Phi(\x, \y)+D_\Phi(\x', \x)-D_\Phi(\x', \y)   =   (\nabla \Phi(\x)-\nabla \Phi(\y))\cdot (\x-\x') \leq 0, \qquad \forall \x' \in \X.
\end{align}
Where the inequality is obtained from the preceding part. This concludes the proof.\end{proof}

\subsection{One-Step Online Mirror Descent Inequality}
\begin{lemma}\label{lemma:gradient_omd_bound}
 Let $\x\in\X$ be arbitrary. At time $t$, given a gradient (or subgradient) $\vec b_t$, the OMD update rule~(\ref{alg:omd}) with state $\x_t$ satisfies
    \begin{align}
    \vec{b}_t \cdot (
        \x_t -\x) \leq \frac{1}{\eta } \parentheses{ D_{\Phi}(\x, \x_t) - D_{\Phi}(\x, \x_{t+1})+  D_\Phi(\x_t, \y_{t+1})} .
\end{align}
\end{lemma}
\begin{proof}
    The OMD algorithm satisfies:
\begin{align}
    \vec{b}_t \cdot (
        \x_t -\x)  &= \frac{1}{\eta} \parentheses{\nabla\Phi(\x_t) - \nabla \Phi(\y_{t+1})} \cdot (\x_t - \x)\\
        &\leq \frac{1}{\eta}\parentheses{ D_{\Phi}(\x, \x_t) - D_{\Phi}(\x, \y_{t+1})+  D_\Phi(\x_t, \x_{t+1})}&\text{(Lemma~\ref{lemma:Pythagorean})}\\
        &\leq \frac{1}{\eta}\parentheses{ D_{\Phi}(\x, \x_t) - D_{\Phi}(\x, \x_{t+1})+  D_\Phi(\x_t, \x_{t+1})}. &\text{(Lemma~\ref{lemma:projection_property})}
\end{align} 
This concludes the proof.
\end{proof}
\section{Specialization to Negative Entropy}
\subsection{Negative Entropy Mirror-Map Regularity}
The mirror map
$\Phi(\vec x) = \sum_{i=1}^n x_i \log x_i$ defined on $\mathcal D = \mathbb{R}^n_{>0}$ satisfies each of the required conditions in Assumption~\ref{asm:mirror_map}, as detailed below.
\begin{itemize}
    \item 
    The probability simplex $\simplex_n$ is contained in the closure $\mathrm{closure}\parentheses{\mathcal  D}$ and intersects $ \mathcal D$ non-trivially ($\simplex_n \cap \mathcal  D \neq \varnothing$).

    \item 
    The function $\Phi$ is strictly convex and continuously differentiable on $\mathcal  D$.
    This follows since, for each coordinate, the mapping $t \mapsto t\log t$ has strictly positive second derivative on $(0,\infty)$.

    \item 
    The gradient of $\Phi$ is given component-wise by $\nabla \Phi(\vec x) = \bigl(1+\log x_1,\,\ldots,\,1+\log x_n\bigr)^\top$. Consequently, $\nabla \Phi$ maps $\mathcal D$ onto $\mathbb{R}^n$: for any $\vec y \in \mathbb{R}^n$, choosing $x_i = e^{y_i-1}$ yields $\nabla \Phi(\vec x)=\vec y$.

    \item 
    As $\vec x$ approaches the boundary of $\mathcal  D$, at least one coordinate satisfies $x_i \to 0^+$, which implies $\log x_i \to -\infty$; hence, $\|\nabla \Phi(x)\|$ diverges as $\vec x$ approaches $\partial \mathcal {D}$.
\end{itemize}
\subsection{Negative-Entropy Mirror Map and Multiplicative Update}

\begin{lemma}\label{eq:update_rule}
    The iterates generated by Algorithm~\ref{alg:omd} and negative entropy mirror map~\eqref{asm:mirror_map} on cost functions $\set{f_s}^t_{t=1}$ and constraints  $\set{g_s}^t_{t=1}$ satisfy the following:
    \begin{align}
        \y_{t+1}  = \parentheses{x_{t,a}\exp\parentheses{ - \eta \tilde{{b}}_{t,a}}}_{a \in \A}, t \in \intv{T}
    \end{align}
\end{lemma}
\begin{proof}
    The gradient of the map $\Phi(\x) = \sum_{a \in \A}x_a\log(x_a)$ is given by $\nabla \Phi(\x) = \parentheses{ \log(x_a) + 1}_{a \in \A}$. Thus, the inverse mapping of $\nabla \Phi: \reals^n_{\geq 0 } \to \reals^n$ is given by 
    \begin{align}
        \parentheses{\nabla}^{-1} \Phi(\x) = \parentheses{\exp(x_a - 1)}_{a \in \A}.
    \end{align}
    Thus, Line 8 in Algorithm~\ref{alg:omd} yields the following update rule 
    \begin{align}
        \y_{t+1} = \parentheses{\exp( \parentheses{ \log(x_{t,a}) + 1}- \eta \tilde b_{t,a}- 1)}_{a \in \A} = \parentheses{\exp( \log(x_{t,a})- \eta \tilde b_{t,a})}_{a \in \A}  =\parentheses{x_{t,a} \exp\parentheses{-\eta \tilde b_{t,a}}}_{a \in \A}.
    \end{align}
    This concludes the proof.
\end{proof}
\subsection{KL-Diameter Bounds on $\simplex_{n,\gamma}$}
\begin{lemma} Consider  $\x_1 = \vec 1/n \in \simplex_n$. The Bregman divergence $\D_\Phi$ under the negative entropy mirror map $\Phi$~\eqref{eq:neg_entropy_mirrormap} has the following upper bounds
\begin{align}
    D_\Phi (\x, \x') &\leq \log\parentheses{{1}/ {\gamma}}, &\text{for $\x,\x' \in \simplex_n$}. \\
    D_\Phi(\x, \x_1) &\leq \log(n), &\text{for $\x \in \simplex_n$}.
\end{align}
\end{lemma}
\begin{proof}
    
    \paragraph{First upper-bound.} Take  $\x, \x' \in \simplex_{n,\gamma}$. The Bregman divergence under the negative entropy over the simplex is the KL divergence:
    \begin{align}
         D_\Phi (\x, \x') &= \sum_{a \in \A} x_a \log\parentheses{\frac{x_a}{x'_a}}= \sum_{a\in\A} x_a \log(x_a) + \sum_{a \in \A} \log\parentheses{\frac{1}{x'_a}} \\
         &\leq \sum_{a \in \A} x_a \log(x_a) + \sum_{a \in \A} x_a \log\parentheses{\frac{1}{\min\set{x'_{a'}: a' \in \A}}} \\
         &=  \sum_{a \in \A} x_a \log(x_a) + \log\parentheses{\frac{1}{\min\set{x'_{a'}: a' \in \A}}} \\
         &\leq\log\parentheses{\frac{1}{\gamma}}.
    \end{align}
    The last inequality is obtained considering that the negative entropy is non-positive over the simplex and  $\min\set{x'_{a'}: a' \in \A} = \gamma $ for any $\x' \in \simplex_{n, \gamma}$. Recall that $\simplex_{n, \gamma} = \simplex_{n } \cap [\gamma, 1]^\A$. The equality holds for a sparse point $\x = \vec e_a$ and a near sparse point $\x'= \sum_{a \in A \setminus \set{a'}} \gamma + (n-1) \gamma \vec e_{a'}  $ for some $a ' \neq a \in \A$.

    \paragraph{Second upper-bound.}  The decision $\x_1$ minimizes the negative entropy $\Phi$, corresponding to a maximum entropy scenario. By the first-order condition~\cite{bubeck2015convex}, we have $-\nabla \Phi(\x_1)(\x - \x_1) \leq 0$ as $\x_1$ minimizes a convex function.
    
    Consequently, we derive the following upper bound for the Bregman divergence:
    \begin{align}
        D_\Phi(\x, \x_1) \leq \sup_{\x \in \simplex_n} \Phi(\x) + \log(n) \leq \log(n).
    \end{align}
    This inequality holds due to the non-positivity of $\Phi(\x)$. The equality holds for sparse points in  $\simplex_n \cap \{0,1\}^n$.
\end{proof}
\subsection{Quadratic Upper Bound for Entropic Bregman Divergence}
\begin{lemma}\label{lemma:bound_Bregman_negative_entropy}
 Let the map $\Phi$ be the  negative entropy~\eqref{eq:neg_entropy_mirrormap}. The Bregman divergence $D_\Phi(\x, \y)$ of  variables $\x \in \simplex_{n, \gamma}$ and $y_{a} \in  x_{a} e^{b_a}$ for $b_a \in \reals_{\leq 0}$ and $a \in \A$  is upper bounded  as follows.
 \begin{align}
     D_\Phi(\x, \y) \leq \frac{1}{2} \sum_{a\in \mathcal{A}}x_{a} \parentheses{b_a}^2,\quad \text{for $\eta \in \reals^n_{\geq 0 }$.}
 \end{align}
\end{lemma}
\begin{proof}
    From the definition of Bregman divergence~\eqref{eq:Bregman_projection} 
    \begin{align}
         D_\Phi(\x, \y) 
         &=  \sum_{a \in \A} x_a \log(x_a) +x_a e^{b_a} (\log(x_a) + b_a) - \sum_{a \in \A} \parentheses{\log(x_a) + b_a + 1} \cdot \parentheses{x_a - x_a e^{b_a}} \\
         &= \sum_{a\in \A} x_a \parentheses{e^{b_a} - b_a  - 1} = \sum_{a \in \A} x_{a} \xi(b_{a}),
    \end{align}
    where $\xi(s) = \exp(s) -s - 1$. Considering that $\xi(s) \leq \frac{1}{2} s^2$ for $s \in \reals_{\leq 0}$ concludes the proof.
\end{proof}
\section{Specialized Bounds for the Bandit Primal-Dual Analysis}
\subsection{Conditional Bregman Divergence Bound}\label{app:bounded_bregman}
\begin{lemma}
 Under the negative entropy mirror map~\eqref{eq:neg_entropy_mirrormap}, the primal iterates of Algorithm~\ref{alg:omd} satisfy
 \begin{align}
\E\interval{D_\Phi(\x_t, \y_{t+1}) \Big|\hist_{t-1}} &\leq    \frac{3\eta^2 n}{2} (1 +  \lambda_t^2 +  \Omega^2).
\end{align}
\end{lemma}
\begin{proof}
At each timeslot $t$, we sample an action a from the action space A with probability $x+{t,a}$. The update rule for OMD is provided in Lemma~\ref{eq:update_rule}. We subsequently compute the expectation of the bound presented in Lemma~\ref{lemma:bound_Bregman_negative_entropy}.
\begin{align}
&\E\interval{D_\Phi(\x_t, \y_{t+1}) \Big|\hist_{t-1}} \leq  
\frac{1}{2}\E\interval{\sum_{a\in \mathcal{A}}x_{t,a} \parentheses{\eta \tilde{\vec{b}}_t \cdot \vec e_a}^2 \Big| \hist_{t-1} } \leq \frac{\eta^2}{2}\sum_{a \in \mathcal{A}} x^2_{t,a}  \parentheses{\frac{3 \Omega^2}{x^2_{t,a_t}} + \frac{3 f^2_{t,a}}{x^2_{t,a}}+ \lambda^2_t \frac{3 g^2_{t,a}}{x^2_{t,a}}}\nonumber \\
&\leq \frac{3\eta^2 n}{2} (1 +  \lambda_t^2 +  \Omega^2).
\end{align}
We used in the second inequality the fact that $\parentheses{x + y + z}^2 \leq 3 x + 3y + 3z$. This concludes the proof.
\end{proof}
\subsection{Primal–Dual Saddle-Point Inequality}\label{app:saddle_inequality}

\begin{lemma}
  Under the negative entropy mirror map~\eqref{eq:neg_entropy_mirrormap}, for any $t \in [T]$ the variables $\lambda_t$ and $\x_t$ in Algorithm~\ref{alg:omd} satisfy the following inequality:
    \begin{align}
       \E \interval{    \Psi_t(\x_t, \lambda) - \Psi_t(\x, \lambda_t)} &\leq \E\interval{\frac{1}{\eta} \parentheses{ D_{\Phi}(\x, \x_t) - D_{\Phi}(\x, \x_{t+1})}+ \frac{1}{2\mu}\parentheses{\parentheses{\lambda_t - \lambda}^2 - \parentheses{\lambda_{t+1} - \lambda}^2}}\nonumber \\
        &+ \E\interval{ \parentheses{\frac{D_\Phi(\x_t, \y_{t+1})}{\eta}  + \frac{\mu}{2}}}, \qquad\text{for $\x \in \simplex_n$, $\lambda \geq 0$.}
    \end{align}
      
\end{lemma}
\begin{proof}
    Given any $\lambda_t \geq 0$ the function $\Psi_t(\,\cdot\,, \lambda_t)$ is linear, so the following holds
    \begin{align}
        \Psi_t(\x, \lambda_t) - \Psi_t(\x_t, \lambda_t) = \nabla_{\x} \Psi_t(\x_t, \lambda_t) \cdot(\x - \x_t).
    \end{align}
    Given any $\x_t \in \simplex_n$ the function $\Psi_t(\x_t, \,\cdot\,)$ is linear,  so the following holds
    \begin{align}
        \Psi_t(\x_t, \lambda) - \Psi_t(\x_t, \lambda_t) = \frac{\partial \Psi_t(\x_t, \lambda_t)}{\partial \lambda} \cdot(\lambda - \lambda_t).
    \end{align}
    Combine the preceding equations to obtain:
    \begin{align}
         \Psi_t(\x, \lambda_t) -\Psi_t(\x_t, \lambda)  =  \nabla_{\x} \Psi_t(\x_t, \lambda_t) \cdot(\x - \x_t)-\frac{\partial \Psi_t(\x_t, \lambda_t)}{\partial \lambda} \cdot(\lambda - \lambda_t) \label{eq:saddle_relation}
    \end{align}

    Note that from the dual update rule we have
    \begin{align}
        (\lambda - \lambda_{t+1})^2 \leq\parentheses{\lambda - \lambda_{t} - \mu g_t(a_t)}^2 \leq \parentheses{\lambda - \lambda_t}^2+\mu^2  \parentheses{g_t(a_t)}^2 - 2 \mu g_t(a_t) (\lambda - \lambda_t).
    \end{align}
    This gives
    \begin{align}
     \frac{\partial \Psi_t(\x_t, \lambda_t)}{\partial \lambda}  (\lambda - \lambda_t) &= g_t(\x) (\lambda - \lambda_t) =\E\interval{g_t(a_t) (\lambda - \lambda_t)\big|\hist_{t-1}} \\
      &\leq \E\interval{\frac{1}{2\mu}\parentheses{(\lambda - \lambda_{t})^2  - (\lambda - \lambda_{t+1})^2 } + \frac{\mu}{2} \parentheses{g_t(a_t)}^2  \big| \hist_{t-1}} 
      \\
      &\leq \E\interval{\frac{1}{2\mu}\parentheses{(\lambda - \lambda_{t})^2  - (\lambda - \lambda_{t+1})^2 } + \frac{\mu}{2}\big| \hist_{t-1}}. 
    \end{align}

    Consider the following:
    \begin{align}
        \E\interval{ \Psi_t(\x, \lambda_t) -\Psi_t(\x, \lambda)}  &=\E\interval{  \nabla_{\x} \Psi_t(\x_t, \lambda_t) \cdot(\x - \x_t)-\frac{\partial \Psi_t(\x_t, \lambda_t)}{\partial \lambda} \cdot(\lambda - \lambda_t)}\\
        &= \E\interval{  \tilde{\vec{b}}_t \cdot(\x - \x_t)-g_t(a_t) \cdot(\lambda - \lambda_t)}\\
         &\leq \E \interval{\frac{1}{2\mu}\parentheses{(\lambda - \lambda_{t})^2  - (\lambda - \lambda_{t+1})^2 } + \frac{\mu}{2}}\nonumber\\ &+ \E \interval{\frac{1}{\eta} \parentheses{D_{\Phi}(\x, \x_t) - D_{\Phi}(\x, \x_{t+1}) + D_\Phi(\x_t, \y_{t+1})}}.
    \end{align}
    The first equality is obtained from the Eq.~\eqref{eq:saddle_relation}. The second equation holds since $\tilde{\vec{b}}_t$ and $g_t(a_t)$ are unbiased estimators of the gradients $\nabla_{\x} \Psi_t(\x_t, \lambda_t)$ and $\frac{\partial  \psi(\x_t, \lambda_t)}{\partial \lambda}$, respectively. The inequality follows from Lemma~\ref{lemma:gradient_omd_bound}.    This concludes the proof.
    
\end{proof}
\subsection{Dual Drift Inequality}

\begin{lemma}\label{lemma:dual_difference}
Under the negative-entropy mirror map~\eqref{eq:neg_entropy_mirrormap}, the iterates of Algorithm~\ref{alg:omd} satisfy, for some $\x^\star\in\simplex_n$ guaranteed by Assumption~\ref{asm:slaters},
    \begin{align}\label{eq:dual_difference}
   &\frac{1}{\mu}  \E\interval{\lambda^2_{t+1} - \lambda^2_t}\leq  2 - \rho \E[\lambda_t] + \frac{1}{\eta} \parentheses{ \E \interval{D_{\Phi}({\x}^\star, \x_t) - D_{\Phi}({\x}^\star, \x_{t+1})}} + \parentheses{ \frac{3\eta n}{2} (1 + \E\interval{ \lambda_t^2} +  \Omega^2)  + \frac{\mu}{2}  }.
    \end{align}
\end{lemma}
\begin{proof}
From Lemma~\ref{lemma:saddle_inequality}, we have    for any $\x \in \simplex_n$:
    \begin{align}
    \E \interval{    \Psi_t(\x_t, 0) - \Psi_t(\x, \lambda_t)} &\leq \E\interval{\frac{1}{\eta} \parentheses{ D_{\Phi}(\x, \x_t) - D_{\Phi}(\x, \x_{t+1})}+ \frac{1}{2\mu}\parentheses{\lambda_t^2 - \lambda_{t+1}^2}} + \E\interval{ \parentheses{\frac{D_\Phi(\x_t, \y_{t+1})}{\eta}  + \frac{\mu}{2}}}
    \end{align}
Rearranging:
    \begin{align}
&\frac{1}{\mu} \E\interval{\lambda^2_{t+1} - \lambda^2_t }  +         \E \interval{\Psi_t(\x_t, 0) - \Psi_t(\x, \lambda_t) } \\
         &\leq  \E \interval{\frac{1}{\eta}\parentheses{ D_{\Phi}(\x, \x_t) - D_{\Phi}(\x, \x_{t+1})} + \parentheses{ \frac{ D_\Phi(\x_t, \y_{t+1}) }{\eta} + \frac{\mu}{2}}^2}\\
         &\leq \E \interval{\frac{1}{\eta}\parentheses{ D_{\Phi}(\x, \x_t) - D_{\Phi}(\x, \x_{t+1})} + \parentheses{ \frac{3\eta n}{2} (1 + \E\interval{ \lambda_t^2} +  \Omega^2)  + \frac{\mu}{2}}}.
    \end{align}
Consider a feasible action ${\x}^\star$ which always exists for  all $t \in [T]$ from Assumption~\ref{asm:slaters}, we note that $g_t({\x}^\star) \leq -\rho$.
\begin{align}
   &\frac{1}{\mu}   \E\interval{\lambda^2_{t+1} - \lambda^2_t} \\
   &\leq\E \interval{f_t(\x) +   \lambda_t g_t(\x_t^\star) - f_t(\x_t)  + \frac{1}{\eta}\parentheses{ D_{\Phi}({\x}^\star, \x_t) - D_{\Phi}({\x}^\star, \x_{t+1})} + \parentheses{\frac{3\eta n}{2} (1 + { \lambda_t^2} +  \Omega^2) + \frac{\mu}{2}  }}\\
    &\leq \E \interval{2 - \rho \lambda_t+ \frac{1}{\eta} \parentheses{ D_{\Phi}({\x}^\star, \x_t) - D_{\Phi}({\x}^\star, \x_{t+1})} + \parentheses{\frac{3\eta n}{2} (1 + { \lambda_t^2} +  \Omega^2)  + \frac{\mu}{2}}}.
\end{align}
This concludes the proof.
\end{proof}

\subsection{Uniform Bound on the Dual Iterates} \label{app:bounded_dual}

\begin{lemma}
 Under the negative entropy mirror map~\eqref{eq:neg_entropy_mirrormap}, the dual variables $\lambda_t$ for $t \in [T]$ in Algorithm~\ref{alg:omd}  are bounded by 
    \begin{align}
        \Omega \triangleq  \frac{\log\parentheses{\gamma^{-1}}}{\rho} \parentheses{\frac{\mu}{\eta}}  +\frac{3  n }{2\rho} \eta + \frac{1}{2\rho} \mu  + \frac{3  n }{\rho} + \frac{2}{\rho}  +  1,\label{eq:upper_bound_omega}
    \end{align}
    for $\eta \leq \Omega^{-2}$ and $\mu \in  \reals_{>0}$. 
\end{lemma}
\begin{proof}
    Our proof adopts a similar logical structure to that presented by~\citet{chen2018bandit} for the Euclidean setting.

    \textbf{First case ($t \leq 1/\mu$).} Consider that $1\leq t \leq \frac{
    1}{\mu}$, then the following holds
    \begin{align}
        \lambda_{t} \leq \lambda_{t-1} + \mu \leq \mu t \leq 1 \leq \Omega.
    \end{align}
    \textbf{Second case ($ 1/\mu \leq t \leq T$).} Assume that $\frac{1}{\mu} \leq t \leq T$, we prove the claim by contradiction. Assume that $T_0$ is the first timeslot for which $\lambda_t > \Omega$. Therefore, it holds that $\lambda_{T_0 - \frac{1}{\mu}} \leq \Omega < \lambda_{T_0}$. This yields
\begin{align}\label{eq:lower_bound_sum}
   \frac{1}{\mu}   \sum^{T_0-1}_{t=T_0 -\frac{1}{\mu}} \E\interval{\lambda^2_{t+1} - \lambda^2_t} = \frac{1}{\mu}\E\interval{\lambda^2_{T_0 } - \lambda^2_{T_0-\frac{1}{\mu}}} > \Omega^2 - \Omega^2 > 0. 
\end{align}
Use Lemma~\ref{lemma:dual_difference} and Lemma~\ref{lemma:path} and sum from $T_0 -\frac{1}{\mu}$ to $T_0-1$ for the points ${\x}^\star$ satisfying the Slater's condition~\ref{asm:slaters} for $t \in [T_0 - \frac{1}{\mu}, T_0]$  to obtain
\begin{align}
    \frac{1}{\mu}   \sum^{T_0-1}_{t=T_0 -\frac{1}{\mu}} \E\interval{\lambda^2_{t+1} - \lambda^2_t} &\leq \frac{2}{\mu}  - \rho \sum^{T_0-1}_{t=T_0 -\frac{1}{\mu}} \E\interval{\lambda_t} + \frac{\log\parentheses{\gamma^{-1}}}{\eta}  +\frac{3 \eta n }{2\mu} +  \frac{3 \eta  n}{2} \sum^{T_0-1}_{t=T_0 -\frac{1}{\mu}}\E\interval{\lambda^2_t} + \frac{3\eta n}{2}\Omega^2 + \frac{1}{2}.
\end{align}
Note that since $\lambda_{T_0} >\Omega$ then we have 
\begin{align}
    \lambda_{T_0 - s} > \Omega  - s \mu, 
\end{align}
where $s \leq \frac{1}{\mu}$.
This gives:
\begin{align}
    \sum^{T_0-1}_{t=T_0 -\frac{1}{\mu}} \E\interval{\lambda_t} > \frac{\Omega}{\mu} -\mu   \sum^{T_0-1}_{t=T_0 - \frac{1}{\mu}} s \geq \frac{\Omega}{\mu} - \frac{1}{\mu}.
\end{align}
Thus, 
\begin{align}
    \frac{1}{\mu}   \sum^{T_0-1}_{t=T_0 -\frac{1}{\mu}} \E\interval{\lambda^2_{t+1} - \lambda^2_t} \leq \frac{2}{\mu}  -\frac{ \rho  \Omega}{\mu} +  \frac{\rho}{\mu}  + \frac{\log\parentheses{\gamma^{-1}}}{\eta}  +\frac{3 \eta n }{2\mu} + \frac{3\eta  n \Omega^2}{\mu} + \frac{1}{2}.
\end{align}
From Eq.~\eqref{eq:lower_bound_sum} and consider $\eta \leq \Omega^{-2}$
\begin{align}
    0 &<\frac{2}{\mu}  -\frac{ \rho  \Omega}{\mu} +  \frac{\rho}{\mu}  + \frac{\log\parentheses{\gamma^{-1}}}{\eta}  +\frac{3 \eta n }{2\mu} + \frac{3\eta  n \Omega^2}{\mu} + \frac{1}{2} \\
    &\leq\frac{2}{\mu}  -\frac{ \rho  \Omega}{\mu} +  \frac{\rho}{\mu}  + \frac{\log\parentheses{\gamma^{-1}}}{\eta}  +\frac{3 \eta n }{2\mu} + \frac{3  n }{\mu} + \frac{1}{2}.
\end{align}
Multiply each side by $\frac{\mu}{\rho}$
\begin{align}
    0 < \frac{2}{\rho}  - \Omega +  1  + \frac{\mu\log\parentheses{\gamma^{-1}}}{\rho\eta}  +\frac{3 \eta n }{2\rho} + \frac{3  n }{\rho} + \frac{\mu}{2\rho}.
\end{align}
This gives 
\begin{align}
    \Omega <   \frac{\log\parentheses{\gamma^{-1}}}{\rho} \parentheses{\frac{\mu}{\eta}}  +\frac{3  n }{2\rho} \eta + \frac{1}{2\rho} \mu  + \frac{3  n }{\rho} + \frac{2}{\rho}  +  1 = \Omega,
\end{align}
which is a contradiction. Then, such $T_0 \leq T$ for which $\lambda_t > \Omega$ does not exist, so $\lambda_t \leq \Omega, \forall t \leq T$. We conclude the proof.

\end{proof}
\subsection{Stability of the Dual Iterates}\label{app:bounded_M}

\begin{lemma}
Let $n \in \mathbb{N}$, $\rho > 0$, and $T \ge 1$. Define $M = 4\parentheses{\frac{3n+2}{\rho} + 1}^2.$ For any scaling constant $c_T \in [1,\sqrt{T}]$, set the learning rate $\eta = \frac{c_T}{M\sqrt{T}}$, the dual learning rate $\mu = \frac{1}{M\sqrt{T}}$, and the restriction threshold $\gamma = T^{-1/2}$. Let $\Omega$ be defined as in Eq.~\eqref{eq:upper_bound_omega}:
\begin{align}
\Omega = \frac{\log\parentheses{\gamma^{-1}}}{\rho}\parentheses{\frac{\mu}{\eta}}
+ \frac{3n}{2\rho}\eta + \frac{1}{2\rho}\mu
+ \frac{3n}{\rho} + \frac{2}{\rho} + 1. \nonumber
\end{align}
Then, for all $T \ge 1$ and all $c_T \in [1,\sqrt{T}]$, the following condition holds:
\begin{align}
\Omega^2 \le \frac{1}{\eta}. \label{eq:condition}
\end{align}
\end{lemma}
\begin{proof}
Let $K = \parentheses{\frac{3n+2}{\rho}+1}$ so that $M = 4K^2$. Substituting $\gamma = T^{-1/2}$, $\eta=\frac{c_T}{M\sqrt{T}}$, and $\mu=\frac{1}{M\sqrt{T}}$ into the definition of $\Omega$ yields
\begin{align}
\Omega
&= \frac{\log\parentheses{\gamma^{-1}}}{\rho}\parentheses{\frac{\mu}{\eta}}
\cdot \frac{3n}{2\rho}\eta + \frac{1}{2\rho}\mu + \frac{3n}{\rho} + \frac{2}{\rho} + 1 \nonumber\\
  &= \frac{\log(T)}{2\rho}\cdot \frac{1}{c_T}
\cdot \parentheses{\frac{3n}{\rho}+\frac{2}{\rho}+1}
\cdot \frac{3n c_T + 1}{2\rho M \sqrt{T}} \nonumber\\
  &= \underbrace{\parentheses{\frac{\log(T)}{2\rho c_T}+K}}\cdot{A(c_T)}  + \underbrace{\frac{3n c_T + 1}{2\rho\sqrt{T}}}\cdot{B(c_T)}\cdot \frac{1}{M}. \nonumber
  \end{align}
  Since $c_T\in[1,\sqrt{T}]$, we have $B(c_T)\le B_{\max}$ with
  \begin{align}
  B_{\max} \triangleq \frac{3n\sqrt{T}+1}{2\rho\sqrt{T}}.
  \end{align}
  Define the upper envelope
  \begin{align}
  \Omega_{\max}(c_T) \triangleq \frac{\log(T)}{2\rho c_T} + K + \frac{B_{\max}}{M},
  \end{align}
  so that $\Omega \le \Omega_{\max}(c_T)$ for all admissible $c_T$. Therefore, it suffices to show
  \begin{align} 
  \Omega_{\max}(c_T)^2 \le \frac{1}{\eta} = \frac{M\sqrt{T}}{c_T}. \label{eq:suff_goal}
  \end{align}

Write $\Omega_{\max}(c_T)$ as $\alpha + \beta/c_T$ where
\begin{align}
\alpha \triangleq K + \frac{B_{\max}}{M}
= K + \frac{3n\sqrt{T}+1}{2\rho\sqrt{T}M},
\qquad
\beta \triangleq \frac{\log(T)}{2\rho}.
\end{align}
Then \eqref{eq:suff_goal} is equivalent to
\begin{align}
\phi(c_T) \triangleq c_T\parentheses{\alpha + \frac{\beta}{c_T}}^2 \le M\sqrt{T}. \label{eq:phi_goal}
\end{align}
Expanding,
\begin{align}
\phi(c) = \alpha^2 c + 2\alpha\beta + \frac{\beta^2}{c}, \qquad c>0.
\end{align}
The function $\phi$ is strictly convex on $(0,\infty)$ because it is the sum of an affine term in $c$ and the strictly convex term $c\mapsto \beta^2/c$. Hence, over the compact interval $[1,\sqrt{T}]$, its maximum is attained at an endpoint:
\begin{align}
\max_{c\in[1,\sqrt{T}]} \phi(c) = \max\set{\phi(1),\phi(\sqrt{T})}.
\end{align}
Thus, it suffices to verify \eqref{eq:phi_goal} for $c_T\in\set{1,\sqrt{T}}$.

\textbf{Case 1: $c_T=\sqrt{T}$.}
Using $\log(T)/\sqrt{T}\le 1$ for $T\ge 1$ and $M=4K^2$, we obtain
\begin{align}
\Omega_{\max}(\sqrt{T})
&= K + \frac{\log(T)}{2\rho\sqrt{T}} + \frac{3n\sqrt{T}+1}{2\rho\sqrt{T}M} \nonumber\\
&\le K + \frac{1}{2\rho} + \frac{3n+1}{2\rho M}
= K + \frac{1}{2\rho} + \frac{3n+1}{8\rho K^2}. \label{eq:omega_case_sqrtT}
\end{align}
To show $\Omega_{\max}(\sqrt{T})^2 \le M = 4K^2$, it is enough to prove $\Omega_{\max}(\sqrt{T}) \le 2K$. By \eqref{eq:omega_case_sqrtT}, it suffices that
\begin{align}
K + \frac{1}{2\rho} + \frac{3n+1}{8\rho K^2} \le 2K
\quad\Longleftrightarrow\quad
\frac{1}{2\rho} + \frac{3n+1}{8\rho K^2} \le K. \label{eq:case1_target}
\end{align}
Multiplying \eqref{eq:case1_target} by $\rho>0$ gives
\begin{align}
\frac{1}{2} + \frac{3n+1}{8K^2} \le \rho K.
\end{align}
Since $\rho K = 3n+2+\rho \ge 3n+2$ and $K^2\ge 1$, we have
\begin{align}
\frac{1}{2} + \frac{3n+1}{8K^2}
\le \frac{1}{2} + \frac{3n+1}{8}
= \frac{3n+5}{8}
\le 3n+2
\le \rho K,
\end{align}
where $\frac{3n+5}{8}\le 3n+2$ holds for all $n\in\mathbb{N}$. Hence $\Omega_{\max}(\sqrt{T})\le 2K$ and thus $\phi(\sqrt{T})\le M\sqrt{T}$.

\textbf{Case 2: $c_T=1$.}
Here $\eta = \frac{1}{M\sqrt{T}}$, so \eqref{eq:phi_goal} becomes $\Omega_{\max}(1)^2 \le M\sqrt{T}$, equivalently $\Omega_{\max}(1)\le 2K,T^{1/4}$ since $M=4K^2$.
Using $M=4K^2$ and $T\ge 1$,
\begin{align}
\Omega_{\max}(1)
&= K + \frac{\log(T)}{2\rho} + \frac{3n\sqrt{T}+1}{2\rho\sqrt{T}M} \nonumber\\
&\le K + \frac{\log(T)}{2\rho} + \frac{3n+1}{2\rho M}
= K + \frac{\log(T)}{2\rho} + \frac{3n+1}{8\rho K^2}. \label{eq:omega_case_1}
\end{align}
Moreover, $\frac{3n+1}{8\rho K^2}\le \frac{3n+1}{8\rho}\le K$ (since $K\ge \frac{3n+2}{\rho}$), so from \eqref{eq:omega_case_1},
\begin{align}
\Omega_{\max}(1) \le \frac{\log(T)}{2\rho} + 2K. \label{eq:omega1_simple}
\end{align}
Using $\log(T)\le 4T^{1/4}$ for $T\ge 1$, we obtain
\begin{align}
\Omega_{\max}(1) \le \frac{2}{\rho}T^{1/4} + 2K. \label{eq:omega1_T14}
\end{align}
Finally, $K=\frac{3n+2}{\rho}+1>\frac{2}{\rho}$ implies $\frac{2}{\rho}T^{1/4}\le K T^{1/4}$, and also $2K \le 2K T^{1/4}$ since $T^{1/4}\ge 1$. Applying both bounds to \eqref{eq:omega1_T14} yields
\begin{align}
\Omega_{\max}(1) \le 2K T^{1/4},
\end{align}
which implies $\phi(1)\le M\sqrt{T}$.

Since $\phi(c_T)\le \max\set{\phi(1),\phi(\sqrt{T})}\le M\sqrt{T}$ for all $c_T\in[1,\sqrt{T}]$, we conclude $\Omega^2\le 1/\eta$ for all valid $T$ and $c_T$.
\end{proof}
\subsection{Telescoping Bound for Drifting Comparators}\label{app:path}

\begin{lemma} \label{lemma:path}
Let $\gamma \in \reals_{> 0}$ and $\set{\x^\star_{t, \gamma}}^T_{t=1}$ be the projection of the comparator sequence~\eqref{eq:comparator} onto $\simplex_{n,\gamma}$.  Under the negative entropy mirror map~\eqref{eq:neg_entropy_mirrormap}, the primal decisions $\set{\x_t}^T_{t=1}$ of Algorithm~\ref{alg:omd} satisfy the following:
    \begin{align}
       \sum^T_{t=t_0}\parentheses{D_\Phi(\x^\star_{t,\gamma}, \x_t) - D_\Phi(\x^\star_{t,\gamma}, \x_{t+1})}  \leq \xi + 2 \log(1/\gamma) \sum^{T-1}_{t=t_0} \norm{ \x^\star_{t+1} - \x^\star_t},
    \end{align}
    where $\xi = \log(n)$ when $t_0 = 1$ and $\xi = \log(1/\gamma)$ for $t_0 >1$.
\end{lemma}
\begin{proof}
Consider the following:
    \begin{align}
   &\sum^T_{t=t_0}\parentheses{D_\Phi(\x^\star_{t,\gamma}, \x_t) - D_\Phi(\x^\star_{t,\gamma}, \x_{t+1})}  \\
   &\leq  \underbrace{D_\Phi(\x^\star_{1,\gamma},\x_{T'})}_{\text{$\leq \xi$}} - \underbrace{D_\Phi(\x^\star_{T, \gamma },\x_{T+1})}_{\text{$\leq 0$}}   +  \sum^{T-1}_{t=t_0} \parentheses{D_\Phi(\x^\star_{t+1, \gamma }, \x_{t+1}) - D_\Phi(\x^\star_{t, \gamma }, \x_{t+1})} \\
   & \leq \xi + \sum^{T-1}_{t=t_0} {\inner{\nabla \Phi(\x^\star_{t+1, \gamma}) - \nabla \Phi(\x_{t+1})}{ \x^\star_{t+1, \gamma} - \x^\star_{t,\gamma}}}  - \underbrace{D_\Phi(\x^\star_{t, \gamma} , \x^\star_{t+1, \gamma})}_{\text{$\geq 0$}}  \\
   & \leq \xi + \sum^{T-1}_{t=t_0} \underbrace{\inner{\nabla \Phi(\x^\star_{t+1, \gamma}) - \nabla \Phi(\x_{t+1})}{ \x^\star_{t+1, \gamma} - \x^\star_{t,\gamma}}}_{\text{Cauchy-Schwarz's Ineq.}}  - \underbrace{D_\Phi(\x^\star_{t, \gamma} , \x^\star_{t+1, \gamma})}_{\text{$\geq 0$}}  \\
   &\leq \xi + \sum^{T-1}_{t=t_0} \underbrace{\norm{\nabla \Phi(\x^\star_{t+1, \gamma}) - \nabla \Phi(\x_{t+1})}_\infty}_{\text{$\leq 2\log(1/\gamma)$}}\norm{ \x^\star_{t+1, \gamma} - \x^\star_{t,\gamma}}_1 \\
   & \leq  \xi + 2  \log(1/\gamma) \sum^{T-1}_{t=t_0} \norm{ \x^\star_{t+1, \gamma} - \x^\star_{t,\gamma}}_1 \leq \xi + 2 \log(1/\gamma)\sum^{T-1}_{t=t_0} \norm{ \x^\star_{t+1} - \x^\star_{t, \gamma}}_1.
\end{align}
This concludes the proof.
\end{proof}

\section{Regret Guarantees}\label{s:formalanalysis}
\subsection{Regret Bound in Terms of Comparator Path Length $P_T$}\label{app:path_length}
\begin{lemma}
    We assume that $\lambda_t \leq \Omega, \forall t \in [T]$. Given a comparator sequence $\set{\vec u_t}^T_{t=1}$ in the simplex $ \simplex_n^T$, we apply Algorithm~\ref{alg:omd} under the negative entropy mirror map~\eqref{eq:neg_entropy_mirrormap}.  For $\beta = n (3/2 + 3 \Omega^2) $, $\eta = \frac{\eta_0}{\sqrt{\beta T}}$ and $\eta_0 > \log(n)/\sqrt{2} $, we establish the following regret guarantee:
    \begin{align}
         \E\interval{ \sum^T_{t=1} f_t(a_t) - f_t(\vec u_{t}) }  &\leq \parentheses{ \frac{\parentheses{\log(n) + 2 \log(1/\gamma) \sum^{T-1}_{t=1}\norm{\vec u_t- \vec u_{t+1}}_1 }  }{\eta_0}+  \eta_0 }\sqrt{\beta T}   \nonumber\\
         &+ 2(1+\Omega)\gamma T + \frac{\mu T}{2}.
    \end{align}
\end{lemma}
\begin{proof}
Consider the following:
    \begin{align}
    \E\interval{\sum^T_{t=1} f_{t} (a_t) - \sum^T_{t=1} f_t(\x^\star_t)}&=  \E\interval{ \sum^T_{t=1}  \E\interval{ f_{t} (a_t)  \Big| \hist_{t-1}}  - \sum^T_{t=1} f_t(\x^{\star}_{ t}) }
=\E\interval{ \sum^T_{t=1} f_t (\x_t) - \sum^T_{t=1} f_t(\x^{\star}_{ t}) }\\
&
 \leq \E\interval{ \sum^T_{t=1} f_t (\x_t) - \sum^T_{t=1} f_t(\x^{\star}_{\gamma, t}) } + 2 (1+\Omega)\gamma  T.
\end{align}
It holds from Lemma~\ref{lemma:saddle_inequality} and Lemma~\ref{lemma:path}
\begin{align}
      &\sum^T_{t=1} \E\interval{\Psi_t(\vec e_{a_t}, \lambda) - \Psi_t(\x^\star_{\gamma,t}, \lambda_t) } = \sum^T_{t=1} \E\interval{\Psi_t(\x_t, \lambda) - \Psi_t(\x^\star_{\gamma,t}, \lambda_t) }  \\
      &\leq  \frac{1}{\eta} \parentheses{\log(n)+ 2 \log(1/\gamma) \sum^{T-1}_{t=1}\norm{\vec u_t- \vec u_{t+1}}_1   + \lambda^2}    + \frac{\mu T}{2} + \sum^T_{t=1}  \E  \interval{\eta^{-1}D_\Phi(\x_t, \y_{t+1}) }
\end{align}
It remains to bound the expectation of $D_\Phi(\x_t, \y_{t+1})$. Lemma~\ref{lemma:bounded_Bregman} and Lemma~\ref{eq:bounded_dual} gives:
\begin{align}
&\E\interval{D_\Phi(\x_t, \y_{t+1}) \Big|\hist_{t-1}}  \leq  \frac{3\eta^2 n}{2} (1 +  2 \Omega^2).
\end{align}
Thus, it holds
\begin{align}
     &\sum^T_{t=1} \E\interval{\Psi_t(\vec e_{a_t}, \lambda) - \Psi_t(\x^\star_{\gamma,t}, \lambda_t) }     \\
     &\leq \frac{1}{\eta} \parentheses{\log(n)+ 2 \log(1/\gamma) \sum^{T-1}_{t=1}\norm{\vec u_t- \vec u_{t+1}}_1   + \lambda^2}    + \frac{\mu T}{2} +  \eta n (3/2 +  3 \Omega^2) T .
\end{align}
Substitute the definition of the Lagrangian with for $\lambda=0$  to obtain the following
\begin{align}
     &\sum^T_{t=1} \E\interval{\Psi_t(\vec e_{a_t}, \lambda) - \Psi_t(\x^\star_{\gamma,t}, \lambda_t) | \hist_{t-1}} = \sum^T_{t=1} \E\interval{\Psi_t(\x_t, \lambda) - \Psi_t(\x^\star_{\gamma,t}, \lambda_t) | \hist_{t-1}} \\
&= \sum^T_{t=1} f_t(\x_t) - f_t(\x^\star_{\gamma,t})  - \sum^T_{t=1}\lambda_t g_t(\x^\star_{\gamma,t})
\end{align}
The comparator point $\x^\star_{t} \in \simplex_{n, t}$, satisfies $g_t(\x^\star_t)\leq 0$, so the projected point $\x^\star_{t, \gamma}$ satisfies $g_t(\x^\star_{t, \gamma}) \leq 2\gamma n$.
This gives 
\begin{align}
     \sum^T_{t=1} f_t(\x_t) - f_t(\x^\star_{\gamma,t})   \leq \frac{1}{\eta} \parentheses{\log(n) + 2 \log(1/\gamma) \sum^{T-1}_{t=1}\norm{\vec u_t- \vec u_{t+1}}_1   }    + \frac{\mu T}{2} + \eta n (3/2 + 3  \Omega^2) T  + 2\gamma n \Omega T.
\end{align}
Replace $\eta$ with $\eta = \frac{\eta_0}{\sqrt{\beta T}}$ in the preceding equation to conclude the proof.
\end{proof}
\subsection{Regret Bound in Terms of Temporal Variation $V_T$} \label{app:temporal_variation}

\begin{lemma}    Assume that $\lambda_t \le \Omega$ for all $t \in [T]$. Fix an arbitrary comparator sequence $\{\vec u_t\}_{t=1}^T \in \simplex_n^T$. We run Algorithm~\ref{alg:omd} with the negative-entropy mirror map in~\eqref{eq:neg_entropy_mirrormap}. Let $\beta = n(3/2 + 3\Omega^2)$ and set $\eta = \eta_0/\sqrt{\beta T}$ with $\eta_0 > \log(n)/\sqrt{2}$. Under these choices, we obtain the following regret guarantee.

    \begin{align}
     \E\interval{ \sum^T_{t=1} f_t(a_t) - f_t(\vec u_{t}) }  &\leq 2 \eta_0 \sqrt{\beta T} + \frac{4 \log(1/\gamma) T V_T }{\eta_0^2 - \log(n)/2} \mathds{1}\parentheses{\sum^{T-1}_{t=1}\norm{\vec u_t- \vec u_{t+1}}_1 > \frac{ \eta_0 ^2 - \log(n)/2}{\log(1/\gamma)} }\nonumber\\
     &+ 2 (1 + \Omega )\gamma  T + \frac{\mu T}{2}.
    \end{align}
\end{lemma}
\begin{proof} 
We adapt the proof of Lemma 2 in \citet{jadbabaie2015online} for our analysis. To begin, we define the following set. Define the following set 
\begin{align}
    \mathcal{U}_T \triangleq \set{\vec u_1, \dots, \vec u_T \in \simplex_n : \sum^{T-1}_{t=1}\norm{\vec u_t- \vec u_{t+1}}_1 \leq \frac{ \eta_0 ^2 - \log(n)/2}{\log(1/\gamma)}, g_t(\vec u_t) \leq 0, t \in [T]}
\end{align}
    and 
    \begin{align}
        \set{\vec u^\star_t}^T_{t=1}\in \argmin_{ \set{\vec u_t}^T_{t=1} \in \mathcal{U}_T}\sum^T_{t=1} f_t(\vec u_t).
    \end{align}

The choice of $\eta_0 > \log(n)/\sqrt{2}$ indicates that the fixed comparator with zero path length is in $\mathcal{U}_T$ so $\mathcal{U}_T \neq \varnothing$, and in term the sequence $\set{\vec u^\star_t}^T_{t=1}$ exists.

Lemma~\ref{lemma:n0} yields the following regret bound:
\begin{align}
    &\E\interval{ \sum^T_{t=1} f_t(a_t) - f_t(\vec u_{t}) }  \leq   \E\interval{ \sum^T_{t=1} f_t(a_t) - f_t(\vec u^\star_{t}) } + \sum^T_{t=1} f_t(\vec u^\star_{t}) -  f_t(\vec u_{t})\\
    &\leq 3 \eta_0 \sqrt{\beta T} + \sum^T_{t=1} f_t(\vec u^\star_{t}) -  f_t(\vec u_{t}) +2 (1+\Omega)\gamma  T + \frac{\mu T}{2} \\
    &\leq 3 \eta_0 \sqrt{\beta T} +\parentheses{\sum^T_{t=1} f_t(\vec u^\star_{t}) -  f_t(\vec u_{t})} \mathds{1}\parentheses{\sum^{T-1}_{t=1}\norm{\vec u_t- \vec u_{t+1}}_1 > \frac{ \eta_0 ^2 - \log(n)/2}{\log(1/\gamma)} } + 2 (1+\Omega)\gamma  T + \frac{\mu T}{2},
\end{align}
where the last step follows from the fact that
\begin{align}
    \sum^T_{t=1} f_t(\vec u^\star_{t}) -  f_t(\vec u_{t}) \leq 0\qquad \text{if $\parentheses{\vec u_1, \dots, \vec u_T} \in \mathcal{U}_T$}.
\end{align}
Divide the time-horizon to $B$ batches of equal lengths, and pick a fixed comparator within each batch. This yields:
\begin{align}
    \sum^T_{t=1} \norm{\vec u_t - \vec u_{t+1}}_1 \leq 2B.
\end{align}
    Take $B = \frac{\eta_0^2 - \log(n)/2}{2 \log(1/\gamma)}$,  $\x^\star_t \in \argmin_{\x \in \X : g_t(\x)\leq 0}\,\, f_t(\x)$, and any fixed $t_k \in\mathcal{T}_k \triangleq  [(k-1)(T/B) + 1,  k (T/B)]$ we have
    \begin{align}
       \sum^T_{t=1} f_t(\vec u^\star_{t}) -  f_t(\vec u_{t}) &\leq \sum^T_{t=1} f_t(\vec u^\star_{t}) -  f_t(\vec x^\star_{t}) \\
       & = \sum^{B}_{k=1} \sum_{t \in \mathcal{T}_k} f_t(\vec u^\star_{t}) - f_t(\vec x^\star_{t})\\
       &\leq  \sum^{B}_{k=1} \sum_{t \in \mathcal{T}_k} f_t(\vec x^\star_{t_k}) - f_t(\vec x^\star_{t})\\
       &\leq \parentheses{\frac{T}{B}} \sum^B_{k=1} \max_{t \in \mathcal{T}_k} \set{f_t(\vec x^\star_{t_k}) - f_t(\vec x^\star_{t})}.
    \end{align}
We now claim that for any $t \in \mathcal{T}_k$, we have
\begin{align}
    f_t(\vec x^\star_{t_k}) - f_t(\vec x^\star_{t}) \leq 2 \sum_{s \in \mathcal{T}_k} \sup_{\x \in \simplex_n} \abs{f_s(\x) - f_{s-1}(\x)}
\end{align}
We prove the claim by contradiction. Assume otherwise then there must be $\hat t_k$ such that
\begin{align}
    f_{\hat t_k}(\vec x^\star_{t_k}) - f_{\hat t_k}(\vec x^\star_{\hat t_k}) > 2 \sum_{s \in \mathcal{T}_k} \sup_{\x \in \simplex_n} \abs{f_s(\x) - f_{s-1}(\x)}
\end{align}
which gives
\begin{align}
    f_t(\x^\star_{\hat t_k}) &\leq  f_{\hat t_k}(\vec x^\star_{\hat t_k}) + \sum_{s \in \mathcal{T}_k} \sup_{\x \in \simplex_n} \abs{f_s(\x) - f_{s-1}(\x)}\\
    &<f_{\hat t_k}(\vec x^\star_{t_k}) -  \sum_{s \in \mathcal{T}_k} \sup_{\x \in \simplex_n} \abs{f_s(\x) - f_{s-1}(\x)} \leq f_t (\vec x^\star_{t_k}).
\end{align}
The preceding relation for $t = t_k$ violates the optimality of  definition of $f_t (\vec x^\star_{t_k})$. Thus, we have 
\begin{align}
     \sum^T_{t=1} f_t(\vec u^\star_{t}) -  f_t(\vec u_{t})  &\leq  \parentheses{\frac{T}{B}} \sum^B_{k=1} \max_{t \in \mathcal{T}_k} \set{f_t(\vec x^\star_{t_k}) - f_t(\vec x^\star_{t})} \\
     &\leq 2\parentheses{\frac{T}{B}} \sum^B_{k=1} \sum_{s \in \mathcal{T}_k} \sup_{\x \in \simplex_n} \abs{f_s(\x) - f_{s-1}(\x)}\\
     &= 2\parentheses{\frac{T}{B}} V_T = \frac{4 \log(1/\gamma) T V_T }{\eta_0^2 - \log(n)/2}.
\end{align}
We conclude the proof.
\end{proof}

\section{Adaptive Meta-Algorithm}\label{s:appenix_meta}
\subsection{Meta algorithm \texttt{MBCOMD}}
\begin{algorithm}[t!]
\caption{\texttt{MBCOMD}: Meta (Agnostic) Bandit-Feedback Mirror Descent with
Time-Varying Constraints.}\label{alg:meta-doubling}
\begin{scriptsize}
\begin{algorithmic}[1]
\Require $K_m=\lceil\log_2 L_m\rceil$, grid for $k\in \intv{K_m}$:  $c_k=2^{k}$, $\mu=\tfrac{1}{M_k\sqrt{L_m}}$, $\eta^{(k)}=\tfrac{c_k}{M_k\sqrt{L_m}}$, $\gamma=\Theta({{L^{-1}_m}})$, and $\gamma^{\mathrm{meta}}_m = \Theta({L_m^{-1/3}})$, $\eta^{\mathrm{meta}}_m = \Theta({L_m^{-1/2}})$  ($M_k$ is selected as in Theorem~\ref{theorem:main})
\State Initialize $\vec{x}^{\mathrm{meta}}_{t_m}=\tfrac{1}{K_m}\mathbf{1}$; for all $k \in \intv{K_m}$: $\x^{(k)}_{t_m}=\tfrac{1}{n}\mathbf{1}$; $\lambda_{t_m}=0$.\Comment{phase $\mathcal{I}_m$ with $L_m = \card{\mathcal{I}_m}$ and $t_m = \min \parentheses{\mathcal{I}_m}$}
\For{$t=t_m,\dots,t_m+L_m-1$}
\State Each expert outputs $\x^{(k)}_{t}\in\simplex_{n,\gamma}$.
\State $\vec{x}^{\mathrm{meta}}_{t}=\sum_k x^{\mathrm{meta},(k)}_{t} \x^{(k)}_{t}$; sample $a_t\sim\vec{x}^{\mathrm{meta}}_{t}$.
\State Observe $f_t(a_t), g_t(a_t)$; set $\hat{\vec f}_{t}=\tfrac{f_t(a_t)}{\vec{x}^{\mathrm{meta}}_{t,a_t}}\vec{e}_{a_t}$, $\hat{\vec g}_{t}=\tfrac{g_t(a_t)}{\vec{x}^{\mathrm{meta}}_{t,a_t}}\vec{e}_{a_t}$.
\State For each $k$: $\texttt{BCOMD}^{(k)}$ (Algorithm~\ref{alg:omd}) step with estimators $\hat{\vec f}_t$, $\hat{\vec g}_t$ and step $\eta^{(k)}$; project onto $\simplex_{n,\gamma}$.
\State Meta update: $\tilde m^{(k)}_{t}=\frac{x^{(k)}{t,a_t}}{x^{\mathrm{meta}}_{t,a_t}} f_t(a_t)$; \quad $y^{\mathrm{meta},(k)}_{t+1}=x^{\mathrm{meta},(k)}_{t} \exp\big(-\eta^{\mathrm{meta}}_m \tilde m^{(k)}_{t}\big)$, $\vec x^{\mathrm{meta}}_{t+1} = \Pi_{\simplex_{n,\gamma_m^{\mathrm{meta}}} } \parentheses{\vec y^\mathrm{meta}_{t+1}}$ (KL projection onto $\simplex_{n,\gamma_m^{\mathrm{meta}}} $).

\EndFor
\end{algorithmic}
\end{scriptsize}
\end{algorithm}
We eliminate the need for prior knowledge of $P_T$ or $V_T$ by employing a 
\emph{meta-algorithm}. The time horizon is partitioned into 
geometrically growing phases $\mathcal I_m$ of length $L_m = 2^{m-1}$, 
for $m = 1, \dots, M = \lceil \log_2 T \rceil$.  In each phase $\mathcal I_m$, we run $K_m = \lceil \log_2 L_m \rceil$ 
\texttt{BCOMD} from a geometric grid. At the beginning of every phase, the dual variable 
is reset.  The meta-learner maintains a mixture distribution over these experts and 
aggregates their predictions. Bandit feedback is broadcast to all experts, 
enabling them to update via entropic mirror descent with appropriate expert-level 
losses (see Algorithm~\ref{alg:meta-doubling}). Formally,

\begin{theorem}[Adaptive regret and violation via doubling]\label{thm:meta}
Under Assumptions~\ref{asm:boundedness} and \ref{asm:slaters}, \textup{\texttt{MBCOMD}} (Algorithm~\ref{alg:meta-doubling} run across phases $m=1,\dots,\lceil\log_2 T\rceil$) achieves
\begin{align}
\mathfrak{R}_T(\textup{\texttt{MBCOMD}}) 
&= \tilde{\mathcal O}\!\left(\min\set{\sqrt{P_T} T^{2/3},\, V_T^{1/3} T^{7/9}}\right), \\
\mathfrak{V}_T(\textup{\texttt{MBCOMD}}) 
&= \tilde{\mathcal O}(\sqrt{T}).
\end{align}
\end{theorem}

Before providing the full proof, we establish the following Lemma for the meta learner.
\subsection{Meta-Learner Regret Within a Phase}
\begin{lemma}\label{lem:meta-phase}                                                  
Within a phase $\mathcal{I}$ of length $L$, Algorithm~\ref{alg:meta-doubling} satisfies
\begin{align}
\mathbb{E}\!\left[\sum_{t\in\mathcal{I}} f_t(a_t)\right] 
- \min_{k\in[K]} \sum_{t\in\mathcal{I}} \x^{(k)}_t\cdot \vec f_t
= \BigO{\sqrt{n \log(K) L}},
\end{align}
where the expectation is taken over the algorithm’s randomness.
\end{lemma}

\begin{proof} We provide the proof in five main steps.

\paragraph{Main inequality.}
Applying Lemma~\ref{lemma:saddle_inequality} to the unconstrained instance 
($\lambda=\lambda_t=\mu=0$) gives, for any $\x^{\mathrm{meta},\star}\in\simplex_K$,
\begin{align}
\E\!\left[\Psi_t(\x^{\mathrm{meta}}_t,0) - \Psi_t(\x^{\mathrm{meta},\star},0)\right]
&\leq 
\underbrace{\E\!\left[\tfrac{1}{\eta^{\mathrm{meta}}}\Big(D_{\Phi}(\x^{\mathrm{meta},\star},\x^{\mathrm{meta}}_t) - D_{\Phi}(\x^{\mathrm{meta},\star},\x^{\mathrm{meta}}_{t+1})\Big)\right]}_{(1)} \\
&\quad + 
\underbrace{\tfrac{1}{\eta^{\mathrm{meta}}}\E\!\left[D_\Phi(\x^{\mathrm{meta}}_t, \y^{\mathrm{meta}}_{t+1})\right]}_{(2)}.
\end{align}

\paragraph{Bounding the divergence term $(1)$.}
The comparator sequence $\set{\x^{\mathrm{meta},\star}}_{t \in \mathcal{I}}$ is fixed. Then by Lemma~\ref{lemma:path}, the telescoping structure implies
\begin{align}
\sum_{t\in\mathcal{I}} \E[(1)] \;\leq\; \frac{1}{\eta^{\mathrm{meta}}}\log(K).
\end{align}

\paragraph{Bounding the variance term $(2)$.}
From Lemma~\ref{lemma:bounded_Bregman}, we have
\begin{align}
\E\!\left[D_\Phi(\x^{\mathrm{meta}}_t, \y^{\mathrm{meta}}_{t+1}) \,\Big|\,\hist_{t-1}\right]
&\leq \tfrac{1}{2}\,(\eta^{\mathrm{meta}})^2 
\E\!\left[\sum_{k} x^{\mathrm{meta}}_{t,k}\, \tilde m_{t,k}^2 \,\Big|\, \hist_{t-1}\right].
\end{align}
For expert $k$, the estimated loss is $\tilde m_{t,k} ={x^{(k)}_{t,a_t}} \,\hat f_t(a_t)$, 
where $\hat f_t(a) = \tfrac{f_t(a_t)}{x^{\mathrm{meta}}_{t,a_t}}\,\mathds 1\{a=a_t\}$ is the importance-weighted estimator.  
Thus, $\tilde m_{t,k} = \x^{(k)}_t \cdot \hat{\vec{f}}_t$ and
\begin{align}
(\tilde m_{t,k})^2 = \Big(\sum_a x^{(k)}_{t,a}\,\hat f_t(a)\Big)^2 \le \sum_a x^{(k)}_{t,a}\,\hat f_t(a)^2. \tag{Jensen}
\end{align}
Averaging over experts:
\begin{align}
\sum_k x^{\mathrm{meta}}_{t,k}\,(\tilde m_{t,k})^2\le \sum_k x^{\mathrm{meta}}_{t,k}\sum_a x^{(k)}_{t,a}\,\hat f_t(a)^2 = \sum_a \Big(\sum_k x^{\mathrm{meta}}_{t,k} x^{(k)}_{t,a}\Big)\hat f_t(a)^2 = \sum_a x^{\mathrm{meta}}_{t,a}\,\hat f_t(a)^2,
\end{align}
since $\vec{x}^{\mathrm{meta}}_t=\sum_k x^{\mathrm{meta}}_{t,k}\x^{(k)}_t$.

Taking expectation over $a_t \sim \vec{x}^{\mathrm{meta}}_t$:
\begin{align}
&\mathbb{E}\!\left[\sum_a x^{\mathrm{meta}}_{t,a}\,\hat f_t(a)^2 \,\Big|\, \hist_{t-1}\right] = \mathbb{E}\!\left[x^{\mathrm{meta}}_{t,a_t}\,\frac{f_t(a_t)^2}{(x^{\mathrm{meta}}_{t,a_t})^2}\,\Big|\, \hist_{t-1}\right] = \mathbb{E}\!\left[\frac{f_t(a_t)^2}{x^{\mathrm{meta}}_{t,a_t}}\,\Big|\, \hist_{t-1}\right] \\
&= \sum_a x^{\mathrm{meta}}_{t,a}\,\frac{f_t(a)^2}{x^{\mathrm{meta}}_{t,a}} = \sum_a f_t(a)^2 \;\leq\; n.
\end{align}
As shown earlier, the variance factor satisfies
\begin{align}
\E\!\left[\sum_{k} x^{\mathrm{meta}}_{t,k}\, \tilde m_{t,k}^2 \,\Big|\, \hist_{t-1}\right] \;\leq\; n.
\end{align}
Hence,
\begin{align}
\sum_{t\in\mathcal{I}} \E[(2)] \;\leq\; \eta^{\mathrm{meta}} n L.
\end{align}

\paragraph{Combining bounds.}
Summing over $t\in\mathcal{I}$, we obtain
\begin{align}
\sum_{t \in \mathcal{I}} \E\!\left[\Psi_t(\x^{\mathrm{meta}}_t,0) - \Psi_t(\x^{\mathrm{meta},\star},0)\right]
\;\leq\; \frac{1}{\eta^{\mathrm{meta}}}\log(K) + \eta^{\mathrm{meta}} n L.
\end{align}
Unrolling the definition of $\Psi_t$ gives
\begin{align}
\sum_{t \in \mathcal{I}} f_t(\x^{\mathrm{meta}}_t) - f_t(\x^{\mathrm{meta},\star}) 
\;\leq\; \frac{1}{\eta^{\mathrm{meta}}}\log(K) + \eta^{\mathrm{meta}} n L.
\end{align}

\paragraph{Optimizing the learning rate.}
Choosing $\eta^{\mathrm{meta}} = \Theta\!\Big(\sqrt{\tfrac{\log(K)}{nL}}\Big)$ yields
\begin{align}
\sum_{t \in \mathcal{I}} f_t(\x^{\mathrm{meta}}_t) - f_t(\x^{\mathrm{meta},\star})
= \BigO{\sqrt{n \log(K) L}}.
\end{align}
This concludes the proof.
\end{proof}
\subsection{Proof of Theorem~\ref{thm:meta}}
\begin{proof} We begin by analyzing the regret of a fixed base learner, then quantify the additional regret introduced by the meta algorithm in searching for the best such learner, and finally combine these results to obtain the overall regret over the horizon $T$.
\paragraph{Fixed Base Learner (Phase $\mathcal{I}$).} For a fixed base learner $k^\star \in [K]$ with a fractional state $\vec x_t$ and intermediate (yet unprojected) state $\vec y_t$ the following holds
\begin{align}
&\E\interval{D_\Phi(\x_t, \y_{t+1}) \Big|\hist_{t-1}} \leq  
\frac{1}{2}\E\interval{\sum_{a\in \mathcal{A}}x_{t,a} \parentheses{\eta \tilde{\vec{b}}_t \cdot \vec e_a}^2 \Big| \hist_{t-1} } \leq \nonumber \\ & \frac{\eta^2}{2}\sum_{a \in \mathcal{A}} x^{\mathrm{meta}}_{t, a} x_{t,a}  \parentheses{\frac{3 \Omega^2}{\parentheses{x^{\mathrm{meta}}_{t,a}}^2} + \frac{3 f^2_{t,a}}{\parentheses{x^{\mathrm{meta}}_{t,a}}^2}+ \lambda^2_t \frac{3 g^2_{t,a}}{\parentheses{x^{\mathrm{meta}}_{t,a}}^2}}
\leq \eta^2 n(3/2 +  3 \Omega^2 ) \parentheses{\frac{1}{\gamma^{\mathrm{meta}}}}.
\end{align}
The factor $\tfrac{1}{\gamma^{\mathrm{meta}}}$ represents the additional cost introduced by the meta algorithm at this stage, which propagates through the analysis and ultimately worsens the regret bounds.
From Lemmas~\ref{lemma:n0} and~\ref{lemma:n02} we have the following regret bound on the base algorithm:
\begin{align}
    \mathfrak{R}_T\parentheses{\texttt{BCOMD}^{(k^{\star})}} = \tilde{\mathcal{O}}\parentheses{\min\set{\frac{P_{\mathcal{I}}}{\eta},\frac{V_{\mathcal{I}} L}{\eta_0^2} } + \frac{\eta L}{\gamma^{\mathrm{meta}}}}.
\end{align}
The learning rate is selected to be $\eta = \eta_0  / c_{\mathcal {I}}$, then we have
\begin{align}
    \mathfrak{R}_T\parentheses{\texttt{BCOMD}^{(k^{\star})}} = \tilde{\mathcal{O}}\parentheses{\min\set{\frac{P_{\mathcal{I}}}{\eta_0} c_{\mathcal {I}},\frac{V_{\mathcal{I}} L}{\eta_0^2} } + \frac{\eta_0 L}{\gamma^{\mathrm{meta}} c_{\mathcal {I}}}}.
\end{align}
Take $c_{\mathcal {I}} = \Theta (L^{2/3})$,  $\eta_0 = \Theta\parentheses{\min\set{\sqrt{P_{\mathcal{I}}}, V_{\mathcal{I}}^{1/3} L^{1/9}}}$, and  $\gamma^{\mathrm{meta}} = \Theta(T^{-1/3})$.
\begin{align}
    \mathfrak{R}_T\parentheses{\texttt{BCOMD}^{(k^{\star})}} = \tilde{\mathcal{O}}\parentheses{\min\set{\frac{P_{\mathcal{I}}}{\eta_0} c_{\mathcal {I}},\frac{V_{\mathcal{I}} L}{\eta_0^2} } + \frac{\eta_0 L}{\gamma^{\mathrm{meta}} c_{\mathcal {I}}}} = \tilde{\mathcal{O}}\parentheses{\min\set{\sqrt{P_{\mathcal{I}}} L^{2/3}, V_{\mathcal{I}}^{1/3} L^{7/9}}}.\label{eq:reg_base_learner}
\end{align}

\paragraph{Meta- and Base-Regret Decomposition and Overall Regret.} By Lemma~\ref{lem:meta-phase} and the previous  result~\eqref{eq:reg_base_learner}, within $\mathcal{I}_m$ we have
\begin{align}
\mathbb{E}\!\left[\sum_{t\in\mathcal{I}_m} f_t(a_t)\right]
&\le
\min_{k} \sum_{t\in\mathcal{I}_m} \x^{(k)}_t\cdot \vec f_t
+ C \sqrt{L_m \log K_m} \\
&\le
\sum_{t\in\mathcal{I}_m} f_t(\x^\star_t) + \tilde{\mathcal{O}}\parentheses{\min\set{\sqrt{P_{\mathcal{I}}} L^{2/3}, V_{\mathcal{I}}^{1/3} L^{7/9}}} +  C \sqrt{L_m \log K_m}.
\end{align}
Summing over $m=0,\dots,M-1$ and noting that the comparator sequences concatenate across phases, we get
\begin{align}
\mathbb{E}\!\Big[\sum_{t=1}^T f_t(a_t)\Big]
- \sum_{t=1}^T f_t(\x^\star_t)
&\le
\sum_{m=0}^{M-1} \tilde{\mathcal O}\left(\min\left\{\sqrt{P_{\mathcal{I}_m}},\, L_m^{2/3},\, V_{\mathcal{I}_m}^{1/3} L_m^{7/9}\right\}\right)  + C \sum_{m=0}^{M-1} \sqrt{L_m \log K_m}.
\label{eq:new-sum-phases}
\end{align}

We bound the two sums separately.

\emph{Path-length branch.} By Cauchy–Schwarz,
\begin{align}
\sum_{m=0}^{M-1} \sqrt{P_{\mathcal{I}_m}}\, L_m^{2/3}
&\le \sqrt{\Bigl(\sum_m P_{\mathcal{I}_m}\Bigr)\Bigl(\sum_m L_m^{4/3}\Bigr)} = \sqrt{P_T \sum_m L_m^{4/3}}  = \mathcal{O}\!\big(\sqrt{P_T}\, T^{2/3}\big).
\end{align}
where we used  $L_m=2^m$ and $T=\sum_m L_m=\Theta(2^M)$, and $\sum_m L_m^{4/3}=\Theta\!\big((2^M)^{4/3}\big)=\Theta(T^{4/3})$.

\emph{Variation branch.} Apply Hölder with $(p,q)=(3,\tfrac{3}{2})$:
\begin{align}
\sum_{m=0}^{M-1} V_{\mathcal{I}_m}^{1/3} L_m^{7/9}
&\le
\Bigl(\sum_m V_{\mathcal{I}_m}\Bigr)^{1/3}
\Bigl(\sum_m L_m^{(7/9)\cdot(3/2)}\Bigr)^{2/3} = V_T^{1/3}\,\Bigl(\sum_m L_m^{7/6}\Bigr)^{2/3}.
\end{align}
Again with $L_m=2^m$,  $\sum_m L_m^{7/6}=\Theta\!\big((2^M)^{7/6}\big)=\Theta(T^{7/6})$, so
\begin{align}
\sum_m V_{\mathcal{I}_m}^{1/3} L_m^{7/9}
= \mathcal{O}\parentheses{V_T^{1/3} T^{7/9}}.    
\end{align}
Therefore,
\begin{align}
\sum_{m=0}^{M-1} \tilde{\mathcal O}\!\left(\min\!\left\{\sqrt{P_{\mathcal{I}_m}}L_m^{2/3},\, V_{\mathcal{I}_m}^{1/3} L_m^{7/9}\right\}\right)
&\le
\tilde{\mathcal O}\!\left(\min\!\left\{\sqrt{P_T} T^{2/3},\, V_T^{1/3} T^{7/9}\right\}\right),\label{eq:branch_A}
\end{align}
since $\sum_m \min\{a_m,b_m\}\le \min\{\sum_m a_m, \sum_m b_m\}$.

\emph{Meta overhead.} With $K_m=\lceil\log_2 L_m\rceil$ and $L_m=2^m$,
\begin{align}
\sum_{m=0}^{M-1} \sqrt{L_m \log K_m}
&\le \sum_{m=0}^{M-1} \sqrt{2^m \log(m+1)}= \Theta\!\big(2^{M/2}\sqrt{\log M}\big)= \Theta\!\big(\sqrt{T \log\log T}\big),\label{eq:branch_B}
\end{align}
which is absorbed into $\tilde{\mathcal O}(\,\cdot\,)$ and is dominated by the main terms for nontrivial $T$.

\emph{4.} Plugging~\eqref{eq:branch_A} and~\eqref{eq:branch_B}  into \eqref{eq:new-sum-phases} gives the stated regret bound
\begin{align}
\mathbb{E}\!\Big[\sum_{t=1}^T f_t(a_t)\Big]
-\min_{\x^\star_t} \sum_{t=1}^T f_t(\x^\star_t)
\;\le\;
\tilde{\mathcal O}\!\left(\min\!\left\{\sqrt{P_T} T^{2/3},\, V_T^{1/3} T^{7/9}\right\}\right)
\end{align}
and the $\sqrt{T\log\log T}$ meta overhead is suppressed in $\tilde{\mathcal O}$.

For the cumulative violation, the per-phase guarantee remains $\tilde{\mathcal O}(\sqrt{L_m})$ by linearity of $g_t(\,\cdot\,)$ and Theorem~\ref{theorem:main}, hence
\begin{align}
\mathbb{E}\!\left[\sum_{t=1}^T g_t(a_t)\right]
&= \sum_{m=0}^{M-1} \tilde{\mathcal O}(\sqrt{L_m})= \tilde{\mathcal O}(\sqrt{T}).
\end{align}
This completes the proof.
\end{proof}

\section{Computational Complexity} \label{s:time_complexity}
\subsection{Time Complexity and Projection Implementation Details}
Algorithm~\ref{alg:omd} instantiated with the negative-entropy mirror map can be implemented in polynomial time. 
In particular, Lines~5--8 and~10 consist of elementary vector operations whose runtime is linear in the number of actions, $\card{\A}=n$. 
The dominant computational cost arises from the projection step onto the truncated simplex $\simplex_{n,\gamma}$. 
Given a membership oracle for $\simplex_{n,\gamma}$, such a projection can be carried out in polynomial time; see, e.g., \citet{hazan2016introduction}. 
Moreover, since the projection problem is convex, it admits efficient iterative solvers that achieve arbitrary accuracy. 
Stronger guarantees are also available: there exist strongly polynomial-time algorithms for this projection problem; see \citet[Theorem~3]{gupta2016solving}.

In our setting, the projection can often be performed in overall linear time by first normalizing $\y$ as $\y'=\y/\norm{\y}_1$ and checking whether $\y'\in\simplex_{n,\gamma}$, i.e., whether $y'_a \ge \gamma$ for all $a\in\A$. 
If this condition holds, then $\y'$ already equals the projected point and no further computation is required.

\section{Comparing Variation Measures}\label{app:incomparability}
\subsection{Incomparability of $P_T$ and $V_T$}\label{s:examples_of_variations}
In this section, we adapt an illustrative construction from \cite{jadbabaie2015online} to show that $V_T$ and $P_T$ are not, in general, comparable quantities, even when all losses lie in $[0,1]$.

We first give an example in which $V_T \ll P_T$. For each round $t \ge 1$, define the loss vector $\vec f_t \in [0,1]^n$ by $\vec f_t = (0, \tfrac{1}{T}, 1, \ldots, 1)$ if $t$ is even, and $\vec f_t = (\tfrac{1}{T}, 0, 1, \ldots, 1)$ if $t$ is odd. Let $\vec x_t^\star \in \arg\min_{x \in \simplex_n} \langle \vec f_t, x\rangle$ be an optimal comparator at round $t$. Since the identity of the minimizer alternates between the first two vertices of the simplex, the path-length satisfies $P_T = \Omega(T)$, whereas the per-round change in the loss vectors is only on the order of $1/T$, implying $V_T = \mathcal{O}(1)$.

Conversely, consider a two-action instance ($n=2$) with losses in $[0,1]^2$ given by $\vec f_t = (0,1)$ if $t$ is even and $\vec f_t = (0,\tfrac{1}{2})$ if $t$ is odd. Here, action $1$ is optimal at every round, so one may choose $\vec x_t^\star$ constant over time and obtain $P_T = \mathcal{O}(1)$, while the loss sequence changes by a constant amount each round, yielding $V_T = \Omega(T)$.

Together, these examples show that neither $V_T$ nor $P_T$ admits a general dominance relationship over the other.